%% file: main.tex
\documentclass[a4paper,fleqn]{cas-dc}
\usepackage[authoryear]{natbib}
\usepackage{amsmath,amssymb,amsfonts}
\usepackage{amsthm}
\usepackage{graphicx}
\usepackage{booktabs}
\usepackage{xspace}
\usepackage{arydshln}
\usepackage{subcaption}
\usepackage{tikz}
\usetikzlibrary{calc}
\usetikzlibrary{
	arrows.meta,
	decorations.pathreplacing,
	calc,
	positioning,
	shapes.arrows
}

\newtheorem{remark}{Remark}
\newtheorem{lemma}{Lemma}
\newtheorem{assumption}{Assumption}

\newcommand{\R}{\mathbb{R}}
\newcommand{\B}{\mathcal{B}}

\newcommand{\D}{\mathcal{D}}

\newcommand{\A}{\mathcal{A}}

\begin{document}
	\let\WriteBookmarks\relax
	\def\floatpagepagefraction{1}
	\def\textpagefraction{.001}
	
	\shorttitle{CS3F: Training-free 3D anomaly localization}
	\shortauthors{Le Gia}
	
	\title[mode=title]{Toward Training-Free Zero-Shot Anomaly Detection in 3D Medical Images: A Batch-Based Approach Using 2D Foundation Models}
	
	\author[1]{Tai {Le Gia}}[orcid=0009-0009-0547-9317]
	\cormark[1]
	\ead{giataile@o.cnu.ac.kr}
	\credit{Conceptualization, Methodology, Software, Validation, Formal analysis, Investigation, Data curation, Visualization, Writing -- original draft, Writing -- review \& editing}
	
	\affiliation[1]{
		organization={Department of Mathematics, Chungnam National University},
		city={Daejeon},
		country={Republic of Korea}
	}
	
	\cortext[cor1]{Corresponding author.}
	
\begin{abstract}
	Zero-shot anomaly detection (ZSAD) is attractive for medical imaging because clinical systems must handle heterogeneous acquisition protocols, changing patient populations, and pathologies for which annotated training data may be unavailable. Most existing zero-shot anomaly detection methods are designed for 2D images, and their direct extension to 3D medical volumes is limited by the scarcity of large-scale volumetric foundation models or by the difficulty of utilizing volumetric context. We propose CS3F, a training-free batch-based framework for ZSAD in 3D medical images using 2D foundation models. Each volume is decomposed along multiple anatomical axes and encoded slice-wise by a 2D vision transformer. These are then converted into localized volumetric tokens by pooling neighboring slice features.  Anomaly scores are obtained from cross-subject mutual similarity: tokens that lack close analogues in other subjects are assigned higher anomaly scores. To reduce the attenuation of focal lesion signals caused by depth pooling, we introduce a coarse-to-fine tokenization strategy that enables fine-resolution volumetric scoring without exhaustive matching. CS3F is evaluated on brain MRI across metastases, glioma, and stroke, as well as validated on lung CT to test generalizability beyond atlas-aligned brain MRI. The results show that frozen 2D foundation models can support anomaly localization in 3D medical images, and that the benefit of fine tokenization depends strongly on lesion contrast and imaging modality.
\end{abstract}
	
		\begin{keywords}
		zero-shot anomaly detection \sep 3D medical imaging \sep foundation models \sep brain MRI \sep Lung CT \sep coarse-to-fine tokenization \sep training-free learning
	\end{keywords}

	\maketitle
	
	\section{Introduction}
Anomaly segmentation plays an important role in medical image analysis, as accurate localization of abnormal regions supports early disease detection, diagnosis, treatment planning, automated screening, and scan quality control. Conventional supervised segmentation methods rely on manually annotated abnormal images and therefore are limited by the availability and diversity of pathological annotations~\citep{kamnitsas2017efficient, Isensee2021, yeung2022unified}. As a result, they may struggle when deployed on abnormalities that differ from those observed during training. To reduce the dependence on anomalous annotations, unsupervised anomaly detection (UAD) methods have been widely explored. These methods learn a normal distribution from normal training data and detect deviations during inference, making them attractive for identifying unseen abnormalities, incidental findings~\citep{Rowley417}, and scan-quality defects~\citep{hendriks2024systematic}. However, standard UAD still requires a sufficiently clean and representative normal training cohort. In medical imaging, such cohorts can be costly to curate, and the training distribution may not match the operating distribution due to differences in acquisition protocols or populations.
	
	These limitations motivate the zero-shot anomaly detection (ZSAD) setting, where anomaly localization is performed without adaptation to a target domain. Recent advances in large-scale pretrained vision models~\citep{CLIP, liu2021swin, DINOv2, dinov3} have enabled promising ZSAD methods for 2D images, including medical images, by exploiting generic feature representations extracted from foundation models. Existing approaches for 2D images can be broadly divided into two categories: text-based methods and batch-based methods. Text-based approaches \citep{APRIL-GAN, WinCLIP, AnomalyCLIP} leverage vision--language models such as CLIP~\citep{CLIP} to estimate abnormality using textual prompts, but often require prompt engineering, adaptation, or additional optimization to achieve stable performance. In contrast, batch-based methods \citep{MuSc, CoDeGraph} operate directly on visual tokens extracted from vision transformers and exploit the intrinsic feature geometry across a batch of subjects (e.g., images). These methods are motivated by a simple statistical observation: normal patches consistently find similar counterparts across images, whereas anomalous patches are comparatively rare and therefore yield higher cross-subject distances. By performing cross-subject mutual similarity searches between local features, batch-based approaches identify anomalous regions without prompts or supervision. 
	
In the 3D setting, the batch-based paradigm translates naturally: a batch of subject volumes can be processed to produce local volumetric tokens for each subject, and anomaly scores can then be obtained by applying cross-subject scoring to these tokens. However, directly realizing this idea remains challenging. Large-scale 3D foundation models for medical imaging remain less mature than their 2D counterparts, making 2D foundation models a practical starting point for deriving volumetric representations. Therefore, the key challenge is not the cross-subject scoring mechanism itself, but how to construct volumetric tokens from 3D medical images using readily available 2D foundation models in a way that preserves 3D contextual information, remains sensitive to focal lesions, and is computationally tractable.
	
	In this work, we propose \textbf{C}ross-\textbf{S}ubject Anomaly Scoring for \textbf{3}D volumes via \textbf{F}oundation models (CS3F), an efficient, training-free framework that adapts the batch-based paradigm to 3D medical volumes by utilizing frozen 2D foundation models. CS3F bridges the gap between 2D models and 3D volumes by encoding 2D slices with a vision transformer and aggregating neighboring slice features into localized volumetric tokens that capture depth-wise anatomical context. To recover the complementary 3D information unavailable in a single orientation, the token construction and scoring pipeline is applied independently across the axial, coronal, and sagittal planes, with the resulting anomaly scores fused for the final volumetric prediction.
	
	A primary challenge in scaling this scoring pipeline to the 3D domain is the quadratic computational complexity inherent in performing cross-subject mutual similarity searches over a massive number of volumetric tokens (often in the millions). To make this computation tractable, we project these volumetric token features into a lower-dimensional space via random projection, which approximately preserves local neighborhood geometry while significantly reducing memory overhead and computational costs. Moreover, fixed-scale depth aggregation introduces another limitation: lesion features can become attenuated when focal lesions occupy only a small fraction of the volumetric aggregation region. To address this limitation, we further introduce a coarse-to-fine (C2F) tokenization strategy that extends scoring to finer depth resolutions while avoiding the prohibitive cost of mutual similarity search. This improves sensitivity to focal lesions while keeping the computational burden manageable. Ultimately, CS3F requires no target-domain supervision, no curated normal training data, and no auxiliary annotated anomaly data.
	
This paper builds upon our previous work~\citep{gia2026trainingfree}, where we introduced batch-based ZSAD for 3D brain MRI. This paper presents a substantially redesigned and expanded framework with axis-wise score fusion, coarse-to-fine multi-scale tokenization, lesion-wise fixed-specificity analysis, and cross-organ validation on lung CT. 

	Our main contributions are:
	
	\begin{itemize}
		\item We propose CS3F, a training-free batch-based framework for zero-shot anomaly detection and segmentation in 3D medical images. CS3F extracts volumetric representations utilizing frozen 2D foundation models and identifies anomalous regions through cross-subject mutual similarity without lesion annotations, curated normal training data, or pathology-specific adaptation.
		
		\item We replace the feature-level multi-axis fusion strategy of our preliminary work~\citep{gia2026trainingfree} with an axis-wise score fusion pipeline. This design substantially reduces memory usage, preserves axis-specific spatial evidence, and supports volumetric inputs with different organ geometries.
		
		\item We introduce coarse-to-fine multi-scale tokenization, motivated
		by a formal analysis of pooling-induced lesion attenuation. Coarse
		scoring provides robust global matching, while fine scoring is
		restricted to spatially relevant regions to improve focal-lesion
		localization with tractable memory and runtime.
		
	\item We evaluate CS3F across brain MRI and lung CT, covering
	metastases, glioma, stroke, and lung cancer. The experiments show that CS3F
	consistently outperforms CLIP-based zero-shot baselines across all
	benchmarks and is competitive with reconstruction-based UAD baselines.
		
	\end{itemize}

\section{Related Work}
\label{sec:related-work}
\paragraph{Unsupervised Anomaly Detection in 3D Medical Imaging.}

Unsupervised anomaly detection (UAD) in 3D medical imaging is dominated by reconstruction-based methods, especially in brain MRI. These methods learn the distribution of normality by reconstructing corrupted or compressed inputs and use reconstruction errors as evidence of anomalies. Existing approaches include autoencoder-based methods \citep{atlason2019unsupervised,baur2021autoencoders,cai2024rethinking}, GAN-based methods \citep{schlegl2019f,KuangLung2020}, and more recent diffusion-based methods \citep{pinaya2022fast,behrendt2024patched,frotscher2025unsupervised}. Although reconstruction-based UAD can achieve strong performance in controlled settings, it remains sensitive to domain gaps between training and testing data caused by scanner hardware, acquisition protocols, and population differences \citep{LIANG2026103763}. Consequently, elevated reconstruction errors may arise from artifacts or distribution mismatches rather than genuine pathology. 

Beyond reconstruction, several works learn patch-level representations of normal anatomy. Student--teacher feature matching has been adapted to 3D MRI, where anomalies are flagged by discrepancies between teacher and student feature responses \citep{schwarz2024patch}. In lung CT, self-supervised learning has been used to learn representations of 3D lung patches for lesion localization \citep{almeida2023coopd,LEE2025103559}. \cite{kim20243d} uses pretrained models to construct a PatchCore-style memory bank \citep{patchcore} from multi-view maximum-intensity projections of segmented lungs. These methods demonstrate the value of patch-level, multi-view, and distance-based feature comparison for volumetric UAD, but they still require either training or a normal reference memory bank.

\paragraph{Vision--Language Models for Zero-Shot Anomaly Detection.}

Zero-shot anomaly detection (ZSAD) aims to eliminate the need for target-domain training data. Vision--language models such as CLIP~\citep{CLIP} have inspired methods that compare visual features against textual prompts describing normal and abnormal states. WinCLIP~\citep{WinCLIP} introduces window-based inference with prompt ensembles for ZSAD. APRIL-GAN~\citep{APRIL-GAN} adapts CLIP by projecting visual features into the joint image--text space after fine-tuning on auxiliary annotated datasets, while AnomalyCLIP~\citep{AnomalyCLIP} learns object-agnostic normal and abnormal prompts to emphasize generic abnormality cues over category semantics. Although these methods remove the need for target-domain training, they are not ideally suited to 3D medical anomaly localization~\citep{CLIP-based-Marzullo}. This limitation arises because CLIP's vision--language alignment is pretrained primarily on natural image--caption pairs, leading to misalignment between the vision and text modalities in the medical domain. Moreover, applying these models in a slice-wise manner to 3D volumes discards through-plane contextual information and frequently yields inconsistent anomaly maps across adjacent slices.

\paragraph{Batch-Based Approaches for Zero-Shot Anomaly Detection.}

Batch-based ZSAD~\citep{ACR, MuSc, CoDeGraph} derives anomaly scores from the statistical structure of an unlabeled test batch rather than from language. MuSc~\citep{MuSc} observes that normal patches consistently find close analogues across other images in the batch, whereas anomalous patches are rarer and exhibit larger nearest-neighbor distances. CoDeGraph~\citep{CoDeGraph} analyzes a key failure mode of this paradigm: when consistent anomalies recur across the batch, they may no longer appear isolated in feature space and can therefore receive erroneously low anomaly scores.

 These methods were primarily developed for 2D industrial anomaly detection. Extending them to 3D medical imaging faces three key obstacles. First, there is no general-purpose 3D foundation model comparable in maturity to 2D models such as DINOv2 or CLIP. Second, volumetric images generate orders of magnitude more tokens than 2D images, rendering quadratic mutual-similarity matching computationally prohibitive. Third, medical abnormalities typically require volumetric context rather than independent slice-wise evidence. CS3F addresses these challenges by constructing localized volumetric tokens from multi-axis 2D foundation-model features and performing efficient cross-subject anomaly scoring tailored to 3D medical images.
\paragraph{Foundation Models for 3D Medical Image Representation.}
Large-scale native 3D medical foundation models have recently emerged for volumetric segmentation, recognition, and disease detection. SuPreM~\citep{li2024abdomenatlas} leverages 673K anatomical masks derived from 5K CT volumes during pre-training. In the head CT domain, FM-CT~\citep{zhu20263d} pretrains on 361,663 non-contrast 3D scans to achieve generalizable disease detection. Other volumetric and promptable foundation models similarly underscore the value of large-scale 3D pre-training~\citep{wang2025sam,blankemeier2024merlin}. Nevertheless, developing these models demands massive volumetric datasets and computational resources.

A complementary and resource-efficient direction reuses strong 2D foundation models for 3D data. RAPTOR~\citep{an2025raptor} constructs training-free volume-level embeddings by aggregating multiplanar slice features extracted from pretrained 2D vision transformers. VoxCor \citep{tombak2026voxcor} similarly builds volumetric representations from frozen 2D models for voxel-wise correspondence tasks. These works establish that 2D foundation models can be lifted to 3D without volumetric pretraining. CS3F adopts this 2D-to-3D philosophy but targets anomaly detection and segmentation.

	\section{Method}
	\label{sec:method}
	\input{overall_pipeline}

	\subsection{Problem Formulation and Framework Overview}
	\label{sec:overview}
	Let $\B=\{V_1,\ldots,V_B\}$ denote a batch of preprocessed 3D medical volumes available at test time, where each volume
	$V_i\in\R^{H\times W\times D}$ depicts the same anatomical structure or organ class, such as brains or lungs. The batch may contain normal, abnormal, or a mixture of both. The goal of zero-shot anomaly detection is to assign a voxel-level anomaly score
	\(
	a_i(x,y,z)\in\R_{\geq 0}
	\)
	to each spatial location, using only the available test-time batch $\B$. Our proposed framework consists of two conceptual components. First, each 3D volume is represented by a fixed set of tokens,
	\[
	C_i=\left\lbrace \mathbf{v}_i^1,\ldots,\mathbf{v}_i^N\right\rbrace\subset\R^d,
	\]
	where each token $\mathbf{v}_i^k$ represents a spatially localized 3D patch. Second, each token is assigned an anomaly score by comparing it with tokens from the other volumes in the batch.
	This design rests on the following \emph{Doppelg\"{a}nger Assumption}.
	
	\begin{assumption}[Doppelg\"{a}nger Assumption]
		\label{assump:doppelganger}
		Within a batch of volumes from the same anatomical domain, normal anatomical structures recur across subjects with relatively consistent appearance and spatial organization. As a result, normal tokens tend to have close cross-subject analogues. Pathological structures, in contrast, are not only less frequent but also more heterogeneous in size, morphology, intensity pattern, and anatomical location. Consequently, pathological tokens are less likely to find close visual analogues in other subjects, even when related pathologies are present elsewhere in the batch.\footnote{This assumption can fail when the same abnormal pattern recurs consistently across many subjects with similar appearance and location, allowing anomalous tokens to find close cross-subject analogues and receive low anomaly scores. This failure mode, termed \emph{consistent anomalies}, and a graph-based mitigation strategy are studied in~\cite{CoDeGraph}. For the 3D medical datasets considered in this work, we did not empirically observe this failure mode.}
	\end{assumption}
	
The overview of CS3F is provided in Figure~\ref{fig:cs3f_overall_pipeline}. In the remainder of this section, we first define the anomaly scoring rule in an abstract token space that is independent of token construction details (Sec.~\ref{sec:scoring}). We then describe the construction of volumetric tokens from 2D foundation-model features via multi-axis slicing and depth pooling (Sec.~\ref{sec:tokenization}). Finally, we analyze the pooling-induced attenuation inherent to depth-pooling tokenization (Sec.~\ref{sec:attenuation}) and introduce the coarse-to-fine (C2F) search strategy that mitigates it, together with its computational analysis (Secs.~\ref{sec:c2f}--\ref{sec:c2f-complexity}). 
	
	\subsection{Batch-Based Zero-Shot Anomaly Scoring}
	\label{sec:scoring}
	
	Let
	\(C_i=\left\lbrace\mathbf{v}_i^1,\ldots,\mathbf{v}_i^{N}\right\rbrace\subset\R^d\)
	denote the token set extracted from volume $V_i$. For a query token $\mathbf{v}\in C_i$, its distance to another volume $V_j$ is defined as the nearest-token distance,
	\begin{equation}
		d(\mathbf{v},C_j)
		=
		\min_{1\leq k\leq N}
		\|\mathbf{v}-\mathbf{v}_j^k\|_2,
		\qquad j\neq i.
		\label{eq:nn-distance}
	\end{equation}
	This cross-subject distance measures how well the query token can be explained by the most similar local representation in the reference volume. Computing Eq.~\eqref{eq:nn-distance} for all reference volumes in the batch produces $B-1$ cross-subject nearest-neighbor distances. Sorting them in ascending order gives the mutual similarity profile,
	\begin{equation*}
		\D_{\B}(\mathbf{v})
		=
		\left[
		d(\mathbf{v})_{(1)},\ldots,d(\mathbf{v})_{(B-1)}
		\right],
		\label{eq:msp}
	\end{equation*}
	where $d(\mathbf{v})_{(t)}$ denotes the $t$-th smallest cross-subject distance. The token-level anomaly score is defined as the average of the $K$ smallest entries:
	\begin{equation}
		a(\mathbf{v})
		=
		\frac{1}{K}
		\sum_{t=1}^{K}
		d(\mathbf{v})_{(t)}.
		\label{eq:token-score}
	\end{equation}
The scoring rule is justified by Assumption~\ref{assump:doppelganger}. Normal tokens should find close counterparts in several other subjects because normal anatomy is repeatedly observed with similar visual and spatial characteristics. Anomalous tokens, by contrast, should lack close cross-subject analogues because pathological regions are less frequent and more variable in size, shape, intensity, and location. Consequently, even its nearest cross-subject matches are expected to remain relatively distant. In practice, we set
\(
K=\lfloor \rho(B-1)\rfloor
\) where \(\rho\in[0.1,0.3]\). This requires a token to have multiple close analogues while reducing sensitivity to individual outliers in reference volumes.
	
	After all token scores are computed, they are rearranged according to their spatial grid locations and interpolated to the input grid to obtain a voxel-level anomaly score map.
	
	\subsection{Volumetric Token Construction with 2D Foundation Models}
	\label{sec:tokenization}
	We now describe the construction of the token sets $C_i$ used in Sec.~\ref{sec:scoring}. An illustration of the tokenization process is provided in Figure~\ref{fig:tokenization_tikz}.
	\subsubsection{Slice Encoding}
	
	Let $f$ denote a 2D vision transformer with input resolution
	\(R\times R\) and patch stride $p$ (e.g., DINOv2 with input size $224\times224$ and patch stride $14$). Each 3D volume is processed along the three anatomical axes,
	\[
	\A=\{\mathrm{sagittal},\mathrm{coronal},\mathrm{axial}\}.
	\]
	For an axis $\ell\in\A$, the volume $V_i$ is decomposed into a sequence of $T_{\ell}$ 2D slices:
	\[
	V_i
	\longrightarrow
	\{S_{i,t}^{(\ell)}\}_{t=1}^{T_{\ell}}.
	\]
	Each slice is resized to match the encoder input resolution $R\times R$ and passed through it:
	\begin{equation}
		\widetilde{S}_{i,t}^{(\ell)}
		=
		\operatorname{Resize}
		\left(
		S_{i,t}^{(\ell)}
		\right),  \,\,
		f\left(\widetilde{S}_{i,t}^{(\ell)}\right)
		=
		\left\{
		\mathbf{f}_{i,(t,u,v)}^{(\ell)}
		\in\R^{D_f}
		\right\}_{u,v=1}^{P}.
		\label{eq:slice-encoding}
	\end{equation}
	Here, $D_f$ is the feature dimension and
	\(P=R/p\)
	is the number of patch tokens along each in-plane dimension. Therefore, each encoded slice produces a fixed $P\times P$ grid of patch features, independent of the original in-plane size of the volume.
	
	\subsubsection{Depth Pooling for 3D Token Construction}
	
	Slice-level patch features obtained from Eq.~\eqref{eq:slice-encoding}
	are aggregated across neighboring slices to form volumetric tokens. Let
	\[
	\mathbf q = (q_{\mathrm{sagittal}}, q_{\mathrm{coronal}}, q_{\mathrm{axial}})
	\]
	denote an axis-wise depth-pooling configuration, where \(q_\ell\) is the
	pooling kernel used when the volume is processed along axis \(\ell\). For
	compactness, we index axis-specific token sets and anomaly maps by
	\((\ell,\mathbf q)\); whenever an expression is evaluated for a fixed axis
	\(\ell\), the scalar pooling size used in that expression is \(q_\ell\).
	
	For simplicity, we select $q_{\ell}$ such that \(q_\ell\) divides \(T_\ell\). Therefore, consecutive
	slices can be grouped into non-overlapping blocks,
	\[
	G_m^{(\ell,\mathbf q)}
	=
	\{(m-1)q_\ell+1,\ldots,mq_\ell\},
	\qquad
	m=1,\ldots,M_\ell(\mathbf q),
	\]
	where
	\[
	M_\ell(\mathbf q)
	=
	\frac{T_\ell}{q_\ell}.
	\]
	For each depth block and in-plane patch location, the corresponding 3D
	token is obtained by average pooling and $\ell_2$ normalization:
	\begin{equation}
		\mathbf{z}_{i,(m,u,v)}^{(\ell,\mathbf q)}
		=
		\frac{1}{q_\ell}
		\sum_{t\in G_m^{(\ell,\mathbf q)}}
		\mathbf{f}_{i,(t,u,v)}^{(\ell)},
		\qquad
		\mathbf{z}_{i,(m,u,v)}^{(\ell,\mathbf q)}
		\leftarrow
		\frac{
			\mathbf{z}_{i,(m,u,v)}^{(\ell,\mathbf q)}
		}{
			\left\|
			\mathbf{z}_{i,(m,u,v)}^{(\ell,\mathbf q)}
			\right\|_2
		}.
		\label{eq:depth-pooling}
	\end{equation}
	Thus, for axis \(\ell\), the volumetric token grid has size
	\[
	M_\ell(\mathbf q)\times P\times P.
	\]
For example, when $V$ is a \(224^3\) volume with \(1\,\mathrm{mm}\) isotropic voxel spacing,
\(R=224\), and \(p=14\), each slice produces a \(16\times16\) patch grid.
If \(q_\ell=14\), tokenization along axis \(\ell\) produces a
\(16\times16\times16\) volumetric token grid, and each token summarizes an
approximate \(14\times14\times14\,\mathrm{mm}^3\) region. If \(q_\ell=2\),
the same axis produces a \(112\times16\times16\) token grid, and each token
summarizes a smaller region, approximately \(2\times14\times14\,\mathrm{mm}^3\).

\input{tokenization}

	\subsubsection{Projection to a Compact Feature Space}
	
	The feature dimension $D_f$ of modern foundation models can be large, making cross-subject pairwise distance computation expensive. We therefore project high-dimensional tokens obtained from \eqref{eq:depth-pooling} into a lower-dimensional feature space using a random projection matrix whose entries are sampled from $\mathcal{N}(0,1/d)$:
	\[
	\mathbf{R}_{\mathrm{proj}}^{(\ell)}
	\in\R^{D_f\times d},
	\qquad
	d\ll D_f.
	\]
	For all tokens extracted along anatomical axis \(\ell\), the projected representation is
	\begin{equation*}
		\mathbf{v}_{i,(m,u,v)}^{(\ell,\mathbf q)}
		=
		\left(\mathbf{R}_{\mathrm{proj}}^{(\ell)}\right)^\top
		\mathbf{z}_{i,(m,u,v)}^{(\ell,\mathbf q)}
		\in\R^d.
		\label{eq:projection}
	\end{equation*}
For each axis \(\ell\), the same projection matrix \(\mathbf{R}_{\mathrm{proj}}^{(\ell)}\) is used for all volumes and all tokens.
	By the Johnson--Lindenstrauss lemma~\citep{johnson1984extensions}, random projections approximately preserve pairwise Euclidean distances with high probability for sufficiently large $d$. This makes the projected features a suitable low-dimensional approximation for the mutual-similarity search in Eq.~\eqref{eq:nn-distance}, while substantially reducing memory usage and computational cost.
	
	\subsubsection{Axis-Wise Scoring and Fusion}
	Given an axis \(\ell\) and pooling configuration \(\mathbf q\), the
	projected features define an axis-specific token set for each volume
	\(V_i\):
	\begin{equation*}
		C_i^{(\ell,\mathbf q)}
		=
		\left\{
		\mathbf{v}_{i,(m,u,v)}^{(\ell,\mathbf q)}
		:
		1\leq m\leq M_\ell(\mathbf q),\;
		1\leq u,v\leq P
		\right\}.
		\label{eq:axis-token-set}
	\end{equation*}
	The batch-based scoring rule in Eq.~\eqref{eq:token-score} is applied
	independently to each axis-specific batch of token sets,
	\(\{C_i^{(\ell,\mathbf q)}\}_{i=1}^{B}\). This produces an axis-specific anomaly map,
	\[
	a_i^{(\ell,\mathbf q)}
	\in
	\mathbb{R}^{M_\ell(\mathbf q)\times P\times P}
	\]
	for each volume \(V_i \in \B\). Each axis-specific anomaly map \(a_i^{(\ell,\mathbf q)}\) is interpolated back to the voxel grid of
	\(V_i\). The multi-axis anomaly map for pooling configuration \(\mathbf q\)
	is then obtained by averaging:
	\begin{equation}
		a_i^{(\mathbf q)}(x,y,z)
		=
		\frac{1}{|\A|}
		\sum_{\ell\in\A}
		a_i^{(\ell,\mathbf q)}(x,y,z).
		\label{eq:axis-fusion}
	\end{equation}
	This multi-axis score fusion aggregates complementary anomaly evidence from orthogonal anatomical planes, allowing signals that are weak or hidden in one view to be reinforced by responses from other views. The fused map \(a_i^{(\mathbf{q})}\) is then used as the output voxel-level anomaly score map. The subject-level anomaly score is obtained by taking the maximum voxel score:
\begin{equation*}
	A_i^{(\mathbf q)}
	=
	\max_{(x,y,z)}
	a_i^{(\mathbf q)}(x,y,z).
	\label{eq:subject-score}
\end{equation*}

	\begin{remark}[Foreground masking]
		The proposed framework does not require foreground masking in general. However,
		when an anatomical foreground map is available, such as a brain mask in
		registered brain MRI or an organ mask in organ-specific CT analysis, it
		can be used to suppress irrelevant background regions. Let
		\(\mathcal{M}_i\) denote the binary foreground mask of volume \(V_i\).
		For each token, we compute the fraction of voxels within its spatial
		support that overlap with \(\mathcal{M}_i\). Only tokens whose foreground
		overlap is greater than a threshold \(\delta\) are included in the computation. This reduces the number of tokens per volume, thereby lowering both memory usage and processing time. After
		interpolation, voxels outside \(\mathcal{M}_i\) are assigned zero anomaly
		score.
	\end{remark}
	
	\subsection{Pooling-Induced Lesion Attenuation}
	\label{sec:attenuation}
	The depth-pooling kernel $q$ controls the spatial support of each volumetric token along the slicing direction. A large pooling kernel produces stable 3D tokens by averaging features across multiple neighboring slices, which helps suppress slice-level noise and improve matching between normal anatomical tokens. However, the same averaging operation can attenuate subtle abnormal signals when the lesion occupies only a small fraction of the pooling window.
	
	To make this effect explicit, we consider a simplified setting in which tokens are analyzed before feature normalization and random projection. This corresponds to a fixed depth block
	of \(q\) slices and a fixed in-plane patch location in
	Eq.~\eqref{eq:depth-pooling}. Suppressing the axis and spatial indices for
	readability, the pooled token is written as
	\begin{equation*}
		\mathbf{z}
		=
		\frac{1}{q}
		\sum_{t=1}^{q}
		\mathbf{u}_t ,
		\label{eq:pooled-token-simple}
	\end{equation*}
	where \(\mathbf{u}_t\) is the slice-level feature at slice \(t\). Let
	\(\bar{\mathbf{z}}\) denote the corresponding pooled token from a normal
	counterpart:
	\begin{equation*}
		\bar{\mathbf{z}}
		=
		\frac{1}{q}
		\sum_{t=1}^{q}
		\bar{\mathbf{u}}_t .
		\label{eq:normal-pooled-token-simple}
	\end{equation*}
	Let \(\mathcal{A}\subseteq\{1,\ldots,q\}\) be the subset of slices affected
	by a local abnormality, and define the lesion occupancy within the pooling
	window as
		\(\alpha
		=
		|\mathcal{A}|/{q}.
		\) We have the following lemma.

	\begin{lemma}[Pooled-token sensitivity]
		\label{lem:pooled_token_sensitivity}
		Assume that the average feature displacement over the anomalous slices $\mathcal{A}$ is
		lower bounded by
		\begin{equation*}
			\left\|
			\frac{1}{|\mathcal{A}|}
			\sum_{t\in\mathcal{A}}
			\left(
			\mathbf{u}_t-\bar{\mathbf{u}}_t
			\right)
			\right\|_2
			\geq
			\Delta_0 ,
			\label{eq:avg-anomalous-displacement}
		\end{equation*}
		and that the residual feature difference outside the anomalous slices is
		bounded by
		\begin{equation*}
			\left\|
			\mathbf{u}_t-\bar{\mathbf{u}}_t
			\right\|_2
			\leq
			\varepsilon,
			\qquad
			\forall t\in\{1,\ldots,q\}\setminus\mathcal{A}.
			\label{eq:normal-residual-bound}
		\end{equation*}
		Then the pooled-token displacement satisfies,
		\begin{equation*}
			\left\|
			\mathbf{z}-\bar{\mathbf{z}}
			\right\|_2
			\geq
			\alpha\Delta_0
			-
			(1-\alpha)\varepsilon .
			\label{eq:pooled-token-sensitivity-bound}
		\end{equation*}
	\end{lemma}
	
	\begin{proof}
		It is a direct consequence of the triangle inequality.
	\end{proof}
Lemma~\ref{lem:pooled_token_sensitivity} shows that depth pooling can still
discriminate anomalous tokens even when the lesion occupancy \(\alpha\) is
small, provided that the feature difference
\(\Delta_0\) is sufficiently large.
Thus, tokenization with a large pooling kernel $q$ can remain effective for large or high-contrast
abnormalities. Conversely,
when \(\Delta_0\) is weak, either due to subtle lesion appearance or the limitation of the foundation encoder $f$, the effective signal
\(\alpha\Delta_0\) may become comparable to the residual variation
\((1-\alpha)\varepsilon\). This attenuation becomes more severe as \(q\)
increases, since lesions now occupy a smaller fraction of the pooling
window. In this regime, it becomes difficult to separate abnormal tokens from normal tokens using the proposed framework.
	
By reducing \(q\) and thus increasing the occupancy \(\alpha\), we can directly mitigate this attenuation effect and preserve stronger local abnormal signals for focal lesions. However, smaller pooling kernels
produce more volumetric tokens, and applying cross-subject mutual scoring exhaustively
at this finer resolution substantially increases the matching runtime. This
resolution--cost trade-off motivates the coarse-to-fine search strategy
introduced next.
	
	\subsection{Coarse-to-Fine Search for Multi-scale Token Scoring}
	\label{sec:c2f}
	\input{c2f}
	To obtain the sensitivity of fine-scale tokens without performing
	exhaustive matching, we use a coarse-to-fine (C2F) search strategy.
	The key idea is that coarse tokens provide a robust and computationally
	efficient way to identify anatomically similar regions across subjects.
	Once such coarse correspondences are found, fine-scale matching can be
	restricted to the fine tokens inside these selected coarse regions.
	In this way, coarse tokens serve as a routing mechanism: they locate
	plausible cross-subject anatomical regions, while fine tokens provide
	the spatial detail needed to detect focal abnormalities.  An illustration of C2F routing is provided in Figure~\ref{fig:c2f_routing}.
	
	Let \(\mathbf q_c\) and \(\mathbf q_f\) denote the coarse- and fine-pooling configurations. For a fixed axis \(\ell\), let
	\(q_{c,\ell}\) and \(q_{f,\ell}\) be the corresponding scalar pooling
	kernels, with
	\begin{equation*}
		q_{f,\ell} < q_{c,\ell},
		\qquad
		q_{c,\ell} = h_\ell q_{f,\ell},
		\label{eq:c2f-scale-relation}
	\end{equation*}
	where \(h_\ell\) is an integer subdivision factor. Under this condition,
	each coarse token along axis \(\ell\) corresponds to exactly \(h_\ell\)
	fine tokens along the depth direction for the same in-plane patch
	location. For a coarse-grid location \(r=(m,u,v)\), we denote the corresponding
	coarse token by
	\(
	\mathbf v_{i,r}^{(\ell,\mathbf q_c)}
	\)
	and its associated fine-token children by
	\begin{equation*}
		\mathcal F_{i,r}^{(\ell)}
		=
		\left\{
		\mathbf v_{i,s}^{(\ell,\mathbf q_f)}
		:
		s\in\operatorname{child}_{\ell}(r)
		\right\},
		\qquad
		|\mathcal F_{i,r}^{(\ell)}| = h_\ell .
		\label{eq:c2f-children}
	\end{equation*}
	Here, 
	\(\operatorname{child}_{\ell}(r)\) denotes the fine-grid locations contained
	within the coarse region \(r\) along axis \(\ell\) of $V_i$. For a query coarse token \(\mathbf v_{i,r}^{(\ell,\mathbf q_c)}\), we first
	find its \(L\) nearest coarse tokens in each reference volume \(V_j\),
	\(j\neq i\):
	\begin{equation*}
		\mathcal N_c(r;i,j)
		=
		\operatorname{TopL}_{r'}
		\left(
		\left\|
		\mathbf v_{i,r}^{(\ell,\mathbf q_c)}
		-
		\mathbf v_{j,r'}^{(\ell,\mathbf q_c)}
		\right\|_2
		\right),
		\label{eq:coarse-neighbors}
	\end{equation*}
	where the notation \(\operatorname{TopL}\) returns the \(L\) nearest coarse-grid
	locations in the reference volume. These locations define the coarse
	regions in \(V_j\) that are most similar to the query region in \(V_i\).
	The corresponding restricted fine-scale search space is then defined as
	\begin{equation*}
		\mathcal G_j^{(\ell)}(i,r)
		=
		\bigcup_{r'\in\mathcal N_c(r;i,j)}
		\mathcal F_{j,r'}^{(\ell)} .
		\label{eq:fine-search-space}
	\end{equation*}
	For a fine query token
	\(
	\mathbf v_{i,s}^{(\ell,\mathbf q_f)}
	\in
	\mathcal F_{i,r}^{(\ell)},
	\)
	its distance to reference volume \(V_j\) is approximated by searching only
	over the restricted candidate set \(\mathcal G_j^{(\ell)}(i,r)\):
	\begin{equation}
		d
		\left(
		\mathbf v_{i,s}^{(\ell,\mathbf q_f)}, V_j
		\right)
		\approx
		\min_{\mathbf v\in\mathcal G_j^{(\ell)}(i,r)}
		\left\|
		\mathbf v_{i,s}^{(\ell,\mathbf q_f)}
		-
		\mathbf v
		\right\|_2
		=
		\widetilde d
		\left(
		\mathbf v_{i,s}^{(\ell,\mathbf q_f)}, V_j
		\right).
		\label{eq:restricted-fine-distance}
	\end{equation}
	This replaces exhaustive search over all fine-grained tokens in \(V_j\) with
	search over the fine children of the most similar coarse regions. The
	approximation quality of Eq.~\eqref{eq:restricted-fine-distance} depends
	on whether the selected coarse neighbors preserve the relevant anatomical
	correspondences. The parameter $L$ directly controls this approximation: increasing $L$ enlarges the restricted fine-scale candidate set and increases the likelihood that it includes the true exhaustive fine-scale nearest neighbor. When \(L=N_c^{(\ell)}\), where \(N_c^{(\ell)}\) is the number of coarse
	tokens along axis \(\ell\), all coarse regions in the reference volume are
	selected, and the restricted search becomes equivalent to full
	fine-scale matching.
	
	Repeating Eq.~\eqref{eq:restricted-fine-distance} over all reference
	volumes produces an approximation of the fine-scale mutual similarity
	vector for the query fine token. As in Eq.~\eqref{eq:token-score}, the
	fine-token anomaly score is computed as the average of the \(K\) smallest
	cross-subject distances:
	\begin{equation*}
		a
		\left(
		\mathbf v_{i,s}^{(\ell,\mathbf q_f)}
		\right)
		=
		\frac{1}{K}
		\sum_{t=1}^{K}
		\widetilde d
		\left(
		\mathbf v_{i,s}^{(\ell,\mathbf q_f)}
		\right)_{(t)} ,
		\label{eq:fine-token-score}
	\end{equation*}
	where
	\(
	\widetilde d
	\left(
	\mathbf v_{i,s}^{(\ell,\mathbf q_f)}
	\right)_{(t)}
	\)
	denotes the \(t\)-th smallest entry of the approximate fine-scale mutual
	similarity profile across the batch $\mathcal{B}$. The resulting fine-token scores are then arranged on the fine-token grid, interpolated back to the input voxel space, and fused using
 Eq.~\eqref{eq:axis-fusion} to obtain the
	fine-resolution anomaly map
	\(a_i^{(\mathbf q_f)}.\)
	
	\paragraph{Multi-scale score fusion.}
	The C2F coarse pass already computes the coarse cross-subject distances
	used for routing. Therefore, the coarse-scale anomaly map
	\(a_i^{(\mathbf q_c)}\) can be obtained with little additional cost by
	applying Eq.~\eqref{eq:token-score} to the same coarse-distance profiles.
	Since the former provides stable anatomical
	matching while the latter preserves focal anomaly evidence, we combine this coarse map with the fine C2F map
	\(a_i^{(\mathbf q_f)}\) to obtain a multi-scale anomaly map,
	\begin{equation}
		a_i^{\mathrm{MS}}(x,y,z)
		=
		\frac{1}{2}
		\left(
		a_i^{(\mathbf q_c)}(x,y,z)
		+
		a_i^{(\mathbf q_f)}(x,y,z)
		\right).
		\label{eq:multiscale-fusion}
	\end{equation}
	The subject-level multi-scale anomaly score is then defined as 
	\begin{equation*}
		A_i^{\mathrm{MS}}
		=
		\max_{(x,y,z)}
		a_i^{\mathrm{MS}}(x,y,z).
		\label{eq:multiscale-subject-score}
	\end{equation*}

	\subsection{Computational Analysis of Coarse-to-Fine Search}
	\label{sec:c2f-complexity}
	
	We compare exhaustive fine-scale matching with C2F matching for one
	query and one reference volume along a single axis $\ell$. For readability, we omit the axis notation \(\ell\) in this subsection.  Let $N_c$ denote the number of coarse tokens per volume. Since each
	coarse token spans $h$ fine tokens along the depth direction, the
	number of fine tokens is $N_f = hN_c$.
	
	In exhaustive fine-scale nearest neighbor search, every fine query token is compared
	with all $N_f$ fine tokens in the reference volume. With projected
	feature dimension $d$, this has a complexity of
	\begin{equation}
		O(N_f^2 d) = O(h^2 N_c^2 d).
		\label{eq:global-fine-cost}
	\end{equation}
C2F matching proceeds in two stages. First, coarse routing computes all
	pairwise coarse distances and retains the $L$ nearest coarse neighbors
	for each coarse query token. This costs $O(N_c^2 d)$, ignoring the
	lower-order top-$L$ selection cost, which is small in the intended
	regime where $L$ is small and $d$ is moderately large. Second, each fine
	query token is matched only against the $Lh$ fine children of its
	selected coarse neighbors, rather than against all $N_f$ fine reference
	tokens. The restricted fine search therefore costs
	\begin{equation*}
		O(N_f \cdot Lh \cdot d) = O(Lh^2 N_c d).
		\label{eq:restricted-fine-cost}
	\end{equation*}
	The total C2F cost is
	\begin{equation}
		O(N_c^2 d) + O(Lh^2 N_c d).
		\label{eq:c2f-cost}
	\end{equation}
	
	Dividing \eqref{eq:global-fine-cost} by \eqref{eq:c2f-cost} gives the
	operation-count reduction,
	\begin{equation*}
		\rho_{\mathrm{C2F}}
		=
		\frac{h^2 N_c^2 d}{N_c^2 d + Lh^2 N_c d}
		=
		\frac{h^2}{1 + Lh^2/N_c}.
		\label{eq:c2f-ratio}
	\end{equation*}
	which approaches the ideal \(h^2\)-fold factor whenever we select L such that
	\begin{equation*}
		L \ll \frac{N_c}{h^2}.
	\end{equation*}
	
However, this reduction is not memory-free. C2F may require retaining both
coarse and fine token representations, together with the coarse routing
indices used for restricted matching.

\section{Experiments}
\label{sec:experiments}

In this section, we evaluate the proposed CS3F framework for anomaly
detection in 3D medical images. The experiments cover both anomaly
classification (AC) and anomaly segmentation (AS). The main experiments
are conducted on brain MRI. We further evaluate the
proposed CS3F on lung CT to examine cross-organ applicability.

We report three operating configurations of the proposed CS3F framework.
\textbf{CS3F-C} uses only the coarse-scale anomaly map
\(a_i^{(\mathbf q_c)}\) in Eq.~\ref{eq:multiscale-fusion} and serves as the efficient configuration, requiring
the least memory and runtime. \textbf{CS3F-F} uses only the fine-scale C2F
anomaly map \(a_i^{(\mathbf q_f)}\) and is used to isolate the effect of
fine tokenization. \textbf{CS3F-MS} combines coarse and fine anomaly maps as
in Eq.~\eqref{eq:multiscale-fusion}, aiming to balance the
stability of coarse tokens with the focal-lesion sensitivity of fine tokens.

\subsection{Data Preparation}
\label{sec:data-preparation}

\paragraph{Brain MRI datasets.}
\vspace{0.8em}
We use IXI~\citep{ixi_dataset} as the healthy control cohort and metastasis BraTS-2025
METS~\citep{brats} as the primary abnormal cohort. Since IXI provides native
T1-weighted (T1w) and T2-weighted (T2w) scans, most of the experiments are conducted on these two modalities. For each modality,
\(20\%\) of IXI and BraTS-METS is reserved for evaluation, while the
remaining \(80\%\) is used only for training baseline methods. The final
mixed test batch contains 180 volumes, consisting of 115 IXI healthy
subjects and 65 BraTS-METS abnormal subjects. The proposed method performs
inference jointly across this full mixed batch.

To evaluate generalization beyond brain metastases, we further use
BraTS-GLI T2w~\citep{bratsgli,menze2014multimodal,bakas2017advancing} and
ATLAS R2.0 T1w~\citep{liew2022large}. BraTS-GLI evaluates glioma, which is
a different tumor type from metastasis, while ATLAS R2.0 evaluates stroke
lesions, which are non-tumor abnormalities. In these experiments, the 115
selected IXI healthy subjects are kept fixed, and the 65 BraTS-METS
abnormal subjects are replaced by abnormal subjects from BraTS-GLI or
ATLAS R2.0.

\paragraph{Brain MRI preprocessing.}
All brain MRI volumes are spatially standardized before inference. For the
BraTS-METS and BraTS-GLI experiments, all images are registered to the
SRI-24 atlas~\citep{rohlfing2010sri24} using CaPTk~\citep{pati2019cancer}, matching the standard pipeline used by BraTS. For the ATLAS R2.0
experiment, all images are registered to MNI-152 space to match the standard
registration  of ATLAS R2.0.

After spatial standardization, each scan is skull-stripped using
HD-BET~\citep{isensee2019automated}. Since skull-stripping produces large, empty borders, the central brain region is cropped to a cube of $156^3$ voxels to remove most of the background. The cropped volume is then
resampled to \(224^3\), histogram-standardized using a fixed reference
template~\citep{nyul2000new}, and normalized to \([0,1]\). For segmentation
annotations, we use a binary lesion setting and treat all voxels with
labels greater than zero as anomalous voxels.

\paragraph{Lung CT datasets.}
To evaluate cross-organ applicability, we conduct an additional lung CT
experiment using pseudo-normal lung CT volumes from LIDC-IDRI
\cite{armato2011lung} and abnormal lung cancer CT volumes from the MSD
Lung dataset~\citep{antonelli2022medical}. The pseudo-normal cohort is
extracted from LIDC-IDRI by selecting scans with no annotated nodules
(\texttt{unblindedReadNodule} \(=0\)) and slice thickness not exceeding
2 mm. This results in 57 pseudo-normal lung CT volumes. The abnormal cohort
contains 62 lung cancer volumes from MSD Lung with lesion annotations.

\paragraph{Lung CT preprocessing.}
Lung CT preprocessing differs from the brain MRI pipeline in two key aspects: CT volumes are not registered to a common atlas, and the target organ occupies only a portion of the full scan volume. Following the strategy used in~\cite{kim20243d}, we first segment the left and right lungs using lungmask~\citep{hofmanninger2020automatic}. 
CT volumes are then resampled to 1 mm isotropic voxel spacing. The intensities are clipped to the Hounsfield unit range $  [-800, 400]  $ and min--max normalized. Because lung sizes vary across subjects, each left- or right-lung crop is further adjusted to a fixed matrix size of $  160 \times 238 \times 336  $ via crop-or-pad operations, without altering the voxel spacing.

Our proposed CS3F is then applied independently to the left and right lung crops. Finally, the full-lung anomaly map and subject-level score are obtained by combining the left- and right-lung results. This pipeline does not require atlas registration, allowing us to evaluate the robustness of the batch-based ZSAD framework under less standardized, cross-organ preprocessing conditions.

\begin{table*}[t]
	\centering
	\caption{Summary of datasets used in the experiments. For the brain
		MRI generalization experiments, the selected IXI healthy cohort is
		kept fixed and the abnormal cohort is replaced by another pathology
		dataset.}
	\label{tab:datasets}
	\begin{tabular}{llllcc}
		\toprule
		Normal cohort & Abnormal cohort & Modality & Anomaly type
		& Normal & Abnormal \\
		\midrule
		IXI & BraTS-METS & T1w & Metastasis & 115 & 65 \\
		IXI & BraTS-METS & T2w      & Metastasis & 115 & 65 \\
		IXI & BraTS-GLI  & T2w      & Glioma     & 115 & 65 \\
		IXI & ATLAS R2.0 & T1w      & Stroke     & 115 & 65 \\
		LIDC-IDRI & MSD Lung & CT    & Lung cancer & 57 & 62 \\
		\bottomrule
	\end{tabular}
\end{table*}

\subsection{Experimental Settings}
\label{sec:experimental-settings}
\vspace{0.8em}
\paragraph{Backbone and feature extraction.}
All experiments use a DINOv2-L/14 encoder as the main backbone. Following previous batch-based ZSAD methods~\citep{MuSc,CoDeGraph}, we use patch features from both shallow and deep transformer layers, specifically layers \(6,12,18,\) and \(24\). For each layer, slice-level patch features are extracted as in Eq.~\eqref{eq:slice-encoding} and then converted into volumetric tokens using the multi-axis tokenization procedure described in Sec.~\ref{sec:tokenization}. The proposed scoring procedure is applied independently to each selected layer, and the four resulting voxel-level anomaly maps are averaged to obtain the final layer-aggregated anomaly map. To keep CS3F efficient, all token features are projected to $d=128$ dimensions. All experiments were conducted on a workstation equipped with a single NVIDIA RTX 4070 Ti Super GPU (16 GB) and 64 GB of system RAM.

\paragraph{Pooling configuration.}
The coarse- and fine-pooling configurations are selected based on the input volume geometry. For brain MRI, all preprocessed volumes have size $224^3$ after resampling. We use
\[
\mathbf{q}_c=(14,14,14), \qquad \mathbf{q}_f=(2,2,2),
\]
where \(\mathbf{q}_c\) denotes the coarse-pooling configuration and
\(\mathbf{q}_f\) denotes the fine-pooling configuration. Under this setting,
each coarse token covers approximately
\(9.75\times9.75\times9.75~\mathrm{mm}^3\), while each fine token covers
approximately \(9.75\times9.75\times1.39~\mathrm{mm}^3\), depending on the
processing axis.

For lung CT, each left or right lung crop has size
\(160 \times 238 \times 336\) after preprocessing, corresponding to the sagittal, coronal, and
axial dimensions in our implementation. We use
\[
\mathbf{q}_c=(8,14,14), \qquad \mathbf{q}_f=(2,2,2).
\]
This adapts the coarse tokenization to the non-cubic lung crop while using the
same fine-pooling depth across all three axes.

\paragraph{Scoring and C2F configuration.}
Unless otherwise stated, the foreground masking threshold is set to
\(\delta=0\). The nearest-neighbor ratio in Eq.~\eqref{eq:token-score}
is set to \(\rho=0.1\). For the coarse-to-fine search, we use \(L=16\) nearest coarse
tokens as the routing budget for restricted fine-scale matching. The overall experimental configurations of the proposed CS3F framework are summarized in Table~\ref{tab:implementation-settings}.

\begin{table}[t]
	\centering
	\footnotesize
	\caption{Implementation settings for brain MRI and lung CT.}
	\label{tab:implementation-settings}
	\begin{tabular}{lll}
		\toprule
		Setting & Brain MRI & Lung CT \\
		\midrule
		Input size
		& \(224 \times 224 \times 224\)
		& \(160 \times 238 \times 336\) \\
		Coarse configuration \(\mathbf{q}_c\)
		& \((14,14,14)\)
		& \((8,14,14)\) \\
		Fine configuration \(\mathbf{q}_f\)
		& \((2,2,2)\)
		& \((2,2,2)\) \\
		Foreground masking \(\delta\)
		& 0
		& 0 \\
		Batch scoring \(\rho\)
		& 0.1
		& 0.1 \\
		C2F routing \(L\)
		& 16
		& 16 \\
		Encoder
		& DINOv2-L/14
		& DINOv2-L/14 \\
		Feature layers
		& \(6,12,18,24\)
		& \(6,12,18,24\) \\
		Projection dimension
		& 128
		& 128 \\
		\bottomrule
	\end{tabular}
\end{table}

\paragraph{Baselines.}
We compare CS3F with representative zero-shot and unsupervised anomaly
detection baselines. For CLIP-based zero-shot anomaly detection, we include
WinCLIP~\citep{WinCLIP}, AnomalyCLIP~\citep{AnomalyCLIP}, and
APRIL-GAN~\citep{APRIL-GAN}. WinCLIP uses a ViT-B/14@240 backbone, is evaluated without fine-tuning, and relies solely on prompt engineering at inference time. Since the official
implementation is not publicly available, we adapt an open-source PyTorch
implementation\footnote{
	\url{https://github.com/zqhang/Accurate-WinCLIP-pytorch}} and extend it to
3D by applying it slice-wise.

We use a ViT-L/14@336 backbone for AnomalyCLIP and APRIL-GAN and fine-tune both models with their official implementations. Additionally, we construct an auxiliary dataset called MedMix, comprising 3,970 2D images (1,874 normal and 2,096 abnormal) with segmentation annotations. MedMix is assembled from liver CT~\citep{bilic2023liver}, retinal OCT~\citep{hu2019automated}, CVC-300~\citep{vazquez2017benchmark}, and CVC-ClinicDB~\citep{bernal2015wm}, and does not include brain MRI or lung CT data . For the liver CT and retinal OCT subsets, we follow the data extraction protocol of~\cite{bmad}.

All CLIP-based baselines are applied slice-wise along the axial, coronal, and sagittal planes. Their voxel-level anomaly maps are first computed for individual slices, then interpolated back to the input volume grid, and finally the three axis-wise maps are averaged. For subject-level anomaly detection, we use the image-level anomaly scores produced by these methods and take the maximum score over all slices from all three axes.

We further include two reconstruction-based UAD baselines trained solely on the T1w or T2w modality of healthy IXI volumes: a 3D denoising autoencoder (DAE)~\citep{pmlr-v172-kascenas22a} and the diffusion-based pDDPM~\citep{behrendt2024patched}. The DAE is implemented as a 3D U-Net with skip connections~\citep{xu2024feasibility}, using a strong Gaussian corruption with $\sigma = 3$, following the configuration described by~\cite{LIANG2026103763}. It is trained on centrally cropped $160^3$ volumes after Z-normalization, using the Adam optimizer with a learning rate of $10^{-3}$ for 100 epochs. For pDDPM, we employ a 2D U-Net backbone~\citep{ronneberger2015u} and apply simplex noise (1000 diffusion noise steps) on patches of size $48\times48$. Training is performed for 1000 epochs using Adam optimization with a learning rate of $10^{-4}$ and a batch size of 32. During inference, the subject-level anomaly score for both methods is computed as the mean reconstruction error averaged across all foreground voxels.

\paragraph{Evaluation metrics.}
For anomaly classification, we report patient-level AUROC (P-AUROC), average
precision (P-AP), and maximum F1 score (P-F1). For anomaly segmentation, we
report voxel-level AUROC (V-AUROC), voxel-level AP (V-AP), Dice, and
intersection over union (IoU). V-AUROC and V-AP are computed over all
foreground organ voxels from all test volumes. Black background voxels
introduced by registration, cropping, or padding are excluded, and all
voxel-level metrics are computed at the input-volume resolution.

Unless otherwise specified, Dice denotes the maximum Dice obtained by sweeping
the threshold over the voxel-level anomaly map on abnormal subjects. IoU is
reported at the same threshold. Since Dice and IoU are not well-defined for
normal subjects with empty ground-truth lesion masks, they are computed
only for abnormal subjects, within the foreground organ region.

\subsection{Anomaly Detection}
\label{sec:anomaly-detection}

We first evaluate subject-level anomaly detection, where the main question
is whether a method can reliably distinguish abnormal from healthy
subjects. Table~\ref{tab:ad_results} reports patient-level AUROC, AP, and
F1 on four brain MRI benchmarks.

As shown in Table~\ref{tab:ad_results}, reconstruction-based UAD baselines achieve near-perfect patient-level performance. This can be attributed to the fact that 3D DAE and pDDPM are trained exclusively on normal IXI volumes and use reconstruction errors aggregated over all foreground voxels as patient-level anomaly scores. Such global scores are highly sensitive to the distribution gap between IXI and the target abnormal cohorts: the models may reconstruct normal regions in IXI test samples with a high accuracy, while producing larger reconstruction errors even for normal-appearing regions in BraTS or ATLAS volumes. This is evidenced by their weaker segmentation performance, as further examined in Sec.~\ref{sec:anomaly-segmentation}.

Among ZSAD methods, the proposed CS3F variants consistently outperform the CLIP-based baselines. More importantly, their performance is stable across datasets: all CS3F variants obtain patient-level AUROC above 84\%, whereas CLIP-based methods vary substantially. For example, AnomalyCLIP reaches nearly 90\% AUROC on BraTS metastases datasets but drops to 63.63\% on the ATLAS R2.0 stroke benchmark. Even though the visual features may contain distinguishable anomalous signals, the image-text alignment, which was primarily optimized on natural images, is unable to consistently capture them, leading to this discrepancy. On the other hand, CS3F avoids possible misalignment between language and vision by relying solely on visual features.

Among the CS3F variants, coarse and multi-scale scoring are most effective
for subject-level detection, with CS3F-MS achieving the best performance on
three out of four datasets. In contrast, CS3F-F generally performs worse
at patient-level ranking. This is expected because the subject-level
score is defined from the maximum response over 3D tokens. Fine tokens
cover smaller spatial supports and are more sensitive to local feature
variation, isolated normal outliers, and nearest-neighbor mismatches.
Such localized responses can perturb the maximum subject score and destabilize
patient-level ranking. Therefore, fine-scale tokens are not primarily
beneficial for anomaly classification; their main role is to preserve
localized anomaly evidence for lesion-wise detection, as analyzed in
Sec.~\ref{sec:c2f_analysis}.

\begin{table*}[t]
	\centering
	\caption{Patient-level anomaly detection results. Values are reported as
		percentages. \textbf{Bold} indicates the best zero-shot result, and
		\underline{underline} indicates the best overall result.}
	\label{tab:ad_results}
	\setlength{\tabcolsep}{3.pt}
	
	\begin{minipage}{0.49\textwidth}
		\centering
		\textbf{BraTS-METS T2w}\\[0.5ex]
		\begin{tabular}{llccc}
			\toprule
			Method & Train & AUROC & AP & F1 \\
			\midrule
			3D DAE & IXI T2w& \underline{99.8} & \underline{99.7} & 96.9 \\
			pDDPM & IXI T2w& 98.5 & 97.6 & 93.7 \\
			AnomalyCLIP & MedMix & 90.6 & 87.0 & 79.3 \\
			APRIL-GAN & MedMix & 30.3 & 27.3 & 53.1 \\
			WinCLIP & -- & 56.3 & 40.7 & 53.1 \\
			\hdashline \\[-1ex]
			CS3F-F & -- & 91.8 & 83.9 & 78.3 \\
			CS3F-C & -- & \textbf{94.9} & 87.8 & \textbf{85.3} \\
			CS3F-MS & -- & 94.7 & \textbf{87.9} & 84.9 \\
			\bottomrule
		\end{tabular}
	\end{minipage}
	\hfill
	\begin{minipage}{0.49\textwidth}
		\centering
		\textbf{BraTS-METS T1w}\\[0.5ex]
		\begin{tabular}{llccc}
			\toprule
			Method & Train & AUROC & AP & F1 \\
			\midrule
			3D DAE & IXI T1w& \underline{100.0} & \underline{100.0} & \underline{100.0} \\
			pDDPM & IXI T1w& 99.90& 99.9 & 98.5 \\
			AnomalyCLIP & MedMix & 88.3 & 82.5 & 77.8 \\
			APRIL-GAN & MedMix & 43.7 & 36.8 & 53.1 \\
			WinCLIP & -- & 73.8 & 56.9 & 63.4 \\
			\hdashline \\[-1ex]
			CS3F-F & -- & 96.0 & 90.9 & 88.2 \\
			CS3F-C & -- & \textbf{97.0} & 92.4 & 89.6 \\
			CS3F-MS & -- & \textbf{97.0} & \textbf{92.5} & \textbf{90.8} \\
			\bottomrule
		\end{tabular}
	\end{minipage}
	
	\vspace{1.0ex}
	
	\begin{minipage}{0.49\textwidth}
		\centering
		\textbf{BraTS-GLI T2w}\\[0.5ex]
		\begin{tabular}{llccc}
			\toprule
			Method & Train & AUROC & AP & F1 \\
			\midrule
			3D DAE & IXI T2w& 99.2 & 98.5 & 94.7 \\
			pDDPM & IXI T2w& \underline{99.8} & \underline{99.6} & \underline{97.7} \\
			AnomalyCLIP & MedMix & 73.5 & 66.0 & 62.9 \\
			APRIL-GAN & MedMix & 62.5 & 52.2 & 54.2 \\
			WinCLIP & -- & 73.8 & 56.9 & 63.6 \\
			\hdashline \\[-1ex]
			CS3F-F & -- & 98.3 & 96.8 & \textbf{94.7} \\
			CS3F-C & -- & 99.0 & 98.3 & 94.6 \\
			CS3F-MS & -- & \textbf{99.2} & \textbf{98.6} & \textbf{94.7} \\
			\bottomrule
		\end{tabular}
	\end{minipage}
	\hfill
	\begin{minipage}{0.49\textwidth}
		\centering
		\textbf{ATLAS R2.0 T1w}\\[0.5ex]
		\begin{tabular}{llccc}
			\toprule
			Method & Train & AUROC & AP & F1 \\
			\midrule
			3D DAE & IXI T1w& 98.5 & 93.9 & {98.5} \\
			pDDPM & IXI T1w& \underline{100.0} & \underline{100.0} & \underline{99.2}\\
			AnomalyCLIP & MedMix & 63.6 & 45.6 & 60.1 \\
			APRIL-GAN & MedMix & 41.6 & 40.3 & 53.1 \\
			WinCLIP & -- & 58.2 & 43.3 & 54.1 \\
			\hdashline \\[-1ex]
			CS3F-F & -- & 84.8 & 73.5 & 72.4 \\
			CS3F-C & -- & 88.2 & 72.8 & \textbf{79.3} \\
			CS3F-MS & -- & \textbf{88.6} & \textbf{75.8} & 76.9 \\
			\bottomrule
		\end{tabular}
	\end{minipage}
\end{table*}

\subsection{Anomaly Segmentation}
\label{sec:anomaly-segmentation}

We next evaluate voxel-level anomaly segmentation using voxel-level AUROC,
AP, Dice, and IoU. Although these
aggregate metrics are naturally biased toward larger lesions, they remain the standard and most direct measures of overall segmentation quality. The results are presented in Table~\ref{tab:as_results}, and the segmentation visualization maps are presented in 
Figs.~\ref{fig:mets_t1_t2_qualitative}--\ref{fig:gli_qualitative}.

The strong AS performance of
CS3F comes from two complementary factors. First, large-scale visual
foundation models provide general-purpose features that can encode rich and localized appearance cues. Second, brain MRI volumes
provide a favorable setting for the Doppelg\"{a}nger Assumption: normal
anatomical structures recur consistently across subjects, whereas
pathological structures are poorly matched. By exploiting these properties, CS3F outperforms reconstruction-based UAD baselines across anomaly-segmentation metrics on the metastases benchmarks, except for voxel-level AP in glioma and stroke, where the diffusion-based pDDPM achieves higher results.

CS3F also achieves superior AS results over weakly performing CLIP-based ZSAD across
all benchmarks. CLIP-based methods compute anomaly scores independently on 2D slices via image--text alignment and then assemble the resulting maps into a 3D volume. This slice-wise strategy frequently produces fragmented and inconsistent segmentations, as evident in Figs.~\ref{fig:mets_t1_t2_qualitative}--\ref{fig:gli_qualitative}. In contrast, CS3F forms local volumetric tokens before anomaly scoring, so
each score is defined from a 3D neighborhood rather than from independently
stitched 2D slice responses in slice-wise approaches, leading to more spatially coherent predictions.

Nevertheless, because CS3F relies on foundation-model features without any adaptation, its performance on low-contrast or highly diffuse lesions is constrained by the backbone's ability to encode them as separable outliers. This limitation is illustrated by Subject~3 in Fig.~\ref{fig:mets_t1_t2_qualitative}. On BraTS-METS T1w, CS3F performs poorly because the annotated metastases have low contrast and blend with adjacent brain tissue. On the corresponding T2w slice of the same subject, the lesion is more conspicuous and thus is accurately segmented. This paired example indicates that the effectiveness of CS3F depends on both the visibility of the anomaly in the input modality and the foundation model's ability to encode it as a separable feature-space outlier. To further examine the role of the encoder, Appendix~\ref{app:backbone_robustness} evaluates CS3F with different 2D ViT backbones, including DINOv3~\citep{dinov3}, CLIP ViT-L/14, and the medical-domain pretrained PMC-CLIP~\citep{lin2023pmc}.

Taken together, although there are still limitations, these results strengthen the evidence that 2D foundation models can be effectively adapted for volumetric medical imaging tasks. The present work extends recent advances in training-free volumetric classification~\citep{an2025raptor} to the more challenging problem of anomaly segmentation. However, standard voxel-level metrics such as AUROC, AP, and Dice are dominated by large lesions and therefore do not fully reveal performance on small- and medium-sized abnormalities. We therefore conduct a dedicated lesion-wise analysis in the following section to examine how CS3F variants detect lesions of varying sizes under a fixed specificity constraint.

\begin{table*}[t]
	\centering
	\caption{Voxel-level anomaly segmentation results. Values are reported as
		percentages. \textbf{Bold} indicates the best zero-shot result, and
		\underline{underline} indicates the best overall result.}
	\label{tab:as_results}
	\setlength{\tabcolsep}{3.0pt}
	
	\begin{minipage}{0.49\textwidth}
		\centering
		\textbf{BraTS-METS T2w}\\[0.5ex]
		\begin{tabular}{llcccc}
			\toprule
			Method & Train & AUROC & AP & Dice & IoU \\
			\midrule
			3D DAE & IXI T2w& 89.6 & 30.5 & 25.0 & 17.3 \\
			pDDPM & IXI T2w& 93.2 & 59.2 & 35.6 & 26.4 \\
			AnomalyCLIP & MedMix & 83.5 & 9.0 & 13.0 & 7.6 \\
			APRIL-GAN & MedMix & 74.1 & 7.0 & 16.1 & 9.8 \\
			WinCLIP & -- & 64.1 & 1.9 & 8.7 & 4.8 \\
			\hdashline \\[-1ex]
			CS3F-F & -- & 97.7 & 60.0 & 39.2 & 29.3 \\
			CS3F-C & -- & \underline{\textbf{98.3}} & \underline{\textbf{68.4}} & 41.2 & 31.9 \\
			CS3F-MS & -- & 98.2 & 67.1 & \underline{\textbf{41.6}} & \underline{\textbf{32.1}} \\
			\bottomrule
		\end{tabular}
	\end{minipage}
	\hfill
	\begin{minipage}{0.49\textwidth}
		\centering
		\textbf{BraTS-METS T1w}\\[0.5ex]
		\begin{tabular}{llcccc}
			\toprule
			Method & Train & AUROC & AP & Dice & IoU \\
			\midrule
			3D DAE & IXI T1w& 69.4 & 3.4 & 7.6 & 4.3 \\
			pDDPM & IXI T1w& 73.3 & 7.6 & 10.6 & 6.2 \\
			AnomalyCLIP & MedMix & 82.3 & 9.7 & 10.6 & 6.1 \\
			APRIL-GAN & MedMix & 64.7 & 4.5 & 11.3 & 6.6 \\
			WinCLIP & -- & 44.9 & 1.3 & 7.0 & 3.8 \\
			\hdashline \\[-1ex]
			CS3F-F & -- & 93.0 & 25.1 & 21.8 & 13.9 \\
			CS3F-C & -- & \underline{\textbf{95.3}} & \underline{\textbf{38.4}} & \underline{\textbf{28.7}} & \underline{\textbf{19.8}} \\
			CS3F-MS & -- & 94.6 & 33.1 & 26.2 & 17.6 \\
			\bottomrule
		\end{tabular}
	\end{minipage}
	
	\vspace{2.ex}
	
	\begin{minipage}{0.49\textwidth}
		\centering
		\textbf{BraTS-GLI T2w}\\[0.5ex]
		\begin{tabular}{llcccc}
			\toprule
			Method & Train & AUROC & AP & Dice  & IoU \\
			\midrule
			3D DAE & IXI T2w& 91.9 & 49.7 & 49.1 & 35.6 \\
			pDDPM & IXI T2w& 92.1 & \underline{63.9} & 58.8 & 45.0 \\
			AnomalyCLIP & MedMix & 80.7 & 16.0 & 25.2 & 15.1 \\
			APRIL-GAN & MedMix & 75.0 & 15.3 & 30.8 & 19.3 \\
			WinCLIP & -- & 44.9 & 1.3 & 7.0 & 3.8 \\
			\hdashline \\[-1ex]
			CS3F-F & -- & 96.7 & 57.9 & 57.7 & 44.1 \\
			CS3F-C & -- & \underline{\textbf{97.4}} & \textbf{62.4} & \underline{\textbf{61.1}} & \underline{\textbf{47.9}} \\
			CS3F-MS & -- & 97.2 & 61.8 & 60.7 & 47.4 \\
			\bottomrule
		\end{tabular}
	\end{minipage}
	\hfill
	\begin{minipage}{0.49\textwidth}
		\centering
		\textbf{ATLAS R2.0 T1w}\\[0.5ex]
		\begin{tabular}{llcccc}
			\toprule
			Method & Train & AUROC & AP & Dice  & IoU \\
			\midrule
			3D DAE & IXI T1w& 86.7 & 12.3 & 13.1 & 7.5 \\
			pDDPM & IXI T1w& 89.6 & \underline{37.8} & 20.0 & 13.3 \\
			AnomalyCLIP & MedMix & 69.2 & 2.1 & 4.2 & 2.2 \\
			APRIL-GAN & MedMix & 71.5 & 2.1 & 8.7 & 4.8 \\
			WinCLIP & -- & 71.7 & 1.4 & 4.1 & 2.2 \\
			\hdashline \\[-1ex]
			CS3F-F & -- & 96.0 & \textbf{29.9} & 30.3 & 20.1 \\
			CS3F-C & -- & 96.0 & 26.4 & 30.2 & 20.5 \\
			CS3F-MS & -- & \underline{\textbf{96.2}} & 28.9 & \underline{\textbf{31.5}} & \underline{\textbf{21.4}} \\
			\bottomrule
		\end{tabular}
	\end{minipage}
\end{table*}

\subsection{Analysis of Multi-Scale Tokenization}
\label{sec:c2f_analysis}
This section evaluates the lesion-wise true positive rate (LTPR) at a fixed true-negative rate (TNR) for
CS3F-C, CS3F-F, and CS3F-MS. The analysis focuses on whether fine-scale C2F
matching improves the detection of small- and medium-sized lesions, and how
much additional computation is required to obtain this benefit.

\subsubsection{Lesion-wise sensitivity under fixed specificity}
\label{sec:multiscale_ltpr}
Global voxel-level metrics such as the Dice score are dominated by large lesions and can obscure method performance on smaller abnormalities. To address this, we report LTPR at a fixed 99\% TNR (LTPR\(_{99}\)), corresponding to a threshold at which at most 1\% of all normal foreground voxels in test datasets are labeled anomalous. This fixed specificity budget enables a fair, clinically relevant comparison of lesion-detection performance across CS3F variants.

Following the ISBI MS lesion segmentation evaluation protocol~\citep{CARASS201777}, we define lesions as 18-connected 3D connected components. All annotated lesions with a diameter smaller than 3 mm are excluded from the lesion evaluation. A ground-truth lesion is considered detected if at least one predicted positive component overlaps with it. Lesions are stratified by volume into small ($<100$ mm$^3$), medium ($100$--$1000$ mm$^3$), and large ($>1000$ mm$^3$) groups, following previous work~\citep{SHANG202690}. This yields 58/66/76 small/medium/large lesions for BraTS-METS, 24/19/74 for BraTS-GLI, and 52/59/66 for ATLAS, respectively. In addition to LTPR at 99\% TNR, we report the Dice score at the same operating point (Dice$_{99}$) and the average number of false-positive connected components per volume across all test volumes (FP/V). The FP/V metric is clinically relevant because high lesion detection sensitivity is only useful in practice when the number of spurious findings remains manageable for radiologist review.

The results are presented in Table~\ref{tab:c2f_ltpr}. As discussed in Sec.~\ref{sec:attenuation}, coarse tokens average features over a larger volumetric
support, which improves the cross-subject matching  between tokens but
attenuates focal lesion signals. Fine tokens use a smaller pooling kernel,
thereby preserving more localized anomalous cues. Table~\ref{tab:c2f_ltpr} demonstrates this effect most clearly on ATLAS-T1w, where CS3F-F increases small-lesion LTPR from 11.5\% to 34.6\% ($3.0\times$) and medium-lesion LTPR from 40.7\% to 74.6\% ($1.8\times$). A comparable improvement appears on BraTS-METS T1w, with small-lesion LTPR rising from 6.9\% to 17.2\% and medium-lesion LTPR doubling from 15.2\% to 30.3\%.

This gain in lesion-wise sensitivity, however, comes at the cost of a higher
false-positive burden. On ATLAS-T1w, CS3F-F increases FP/V from 2.7 to
11.9, while on BraTS-METS T1w it increases FP/V from 2.1 to 7.8. This
reflects the opposite side of the fine-grained tokenization trade-off:
reducing the pooling support also reduces contextual aggregation. As a
result, the output tokens become more sensitive to local anatomical
variation, imperfect registration, and heterogeneous normal tissue
appearance.  This trade-off is visible in the qualitative results. Subjects~2 and~3 in Fig.~\ref{fig:atlas_qualitative} show that CS3F-F can recover small lesion components that are omitted by the coarse variant. At the same time, Fig.~\ref{fig:mets_t1_t2_qualitative} shows that the fine variant is accompanied by more scattered false-positive predictions than the coarse variant, especially for the T1w modality of BraTS-METS.

The multi-scale CS3F-MS provides a compromise between these two regimes. On ATLAS-T1w, CS3F-MS improves
medium-lesion LTPR over CS3F-C from 40.7\% to 55.9\%, while substantially reducing FP/V
from 11.9 for CS3F-F to 4.2. On BraTS-GLI T2w, CS3F-MS improves
medium-lesion LTPR from 5.3\% to 15.8\%, with only a moderate FP/V
increase from 1.5 to 2.0. Therefore, multi-scale scoring is introduced as a controlled trade-off: it improves lesion-wise sensitivity
relative to coarse scoring, but does not fully inherit the false-positive
burden of fine-only scoring.

The benefit of adding fine tokens also depends on modality and pathology.
Fine tokenization is most effective when the backbone encoder produces separable lesion representations, as characterized by $\Delta_0$ in
Lemma~\ref{lem:pooled_token_sensitivity}. This explains why the
performance on BraTS-METS T2w remains weak in absolute terms. When lesion and normal tissues are encoded similarly by the foundation models, reducing the pooling kernel can increase lesion occupancy within each token, but it cannot completely overcome the lack of a strong discriminative signal. To further examine this interpretation, we analyze in Appendix~\ref{app:lesion_separability} the classification difficulties of small- and medium-sized lesion features against normal-tissue features across modalities (including T1w, post-contrast T1c and T2w) and pathologies. Collectively, these results indicate that C2F improves lesion signals in CS3F's tokenization process, but its benefit is limited by the lesion sensitivity of the foundation model.

\begin{table}[t]
	\centering
	\caption{Lesion-wise analysis at 99\% true negative rate. S-, M-, and
		L-LTPR\(_{99}\) denote lesion-wise true positive rates for small, medium, and
		large lesions, respectively. FP/V denotes the average number of false-positive connected
		components per volume. Values for LTPR and Dice\(_{99}\) are reported as
		percentages.}
	\label{tab:c2f_ltpr}
	\scriptsize
	\setlength{\tabcolsep}{2.9pt}
	\renewcommand{\arraystretch}{0.95}
	\begin{tabular}{llccccc}
		\toprule
		Dataset & Method & S-LTPR\(_{99}\) & M-LTPR\(_{99}\) & L-LTPR\(_{99}\) & Dice\(_{99}\) & FP/V \\
		\midrule
		\multirow{3}{*}{METS-T2w}
		& CS3F-F  & 10.3 & 9.1  & 86.8 & 39.2 & 10.2 \\
		& CS3F-C  & 6.9  & 1.5  & 79.0 & 41.2 & 2.6  \\
		& CS3F-MS & 6.9  & 4.6  & 81.6 & 41.6 & 3.9  \\
		\midrule
		\multirow{3}{*}{METS-T1w}
		& CS3F-F  & 17.2 & 30.3 & 81.6 & 19.2 & 7.8 \\
		& CS3F-C  & 6.9  & 15.2 & 75.0 & 27.3 & 2.1 \\
		& CS3F-MS & 12.1 & 19.7 & 75.0 & 24.4 & 3.3 \\
		\midrule
		\multirow{3}{*}{GLI-T2w}
		& CS3F-F  & 8.3  & 21.1 & 96.0 & 53.8 & 5.4 \\
		& CS3F-C  & 8.3  & 5.3  & 96.0 & 56.9 & 1.5 \\
		& CS3F-MS & 8.3  & 15.8 & 94.6 & 56.9 & 2.0 \\
		\midrule
		\multirow{3}{*}{ATLAS-T1w}
		& CS3F-F  & 34.6 & 74.6 & 98.5 & 30.1 & 11.9 \\
		& CS3F-C  & 11.5 & 40.7 & 95.5 & 30.0 & 2.7  \\
		& CS3F-MS & 15.4 & 55.9 & 98.5 & 31.5 & 4.2  \\
		\bottomrule
	\end{tabular}
\end{table}

\subsubsection{Computational overhead of multi-scale scoring}
\label{sec:multiscale_efficiency}

Table~\ref{tab:multiscale_efficiency} reports runtime and memory statistics for the full test batch of 180 volumes, with \(L=16\) nearest coarse tokens used for C2F searching. Tokenization (includes batch preparation, encoder forward passes, and volumetric token construction) dominates the runtime for both variants, accounting for nearly 85\% of the total time in the coarse setting. CS3F-C processes each volume in 3.2\,s with a peak VRAM of 5.5\,GB, making it well-suited for consumer-grade GPUs. Adding multi-scale scoring in CS3F-MS increases this to  5.5\,s per volume and 8.4\,GB VRAM. The primary contributor to the added cost is the fine-token matching stage, where the cross-subject scoring time for the full batch rises more than sixfold, from 60.7\,s to 398.9\,s.

Importantly, the C2F routing keeps the CS3F-MS overhead moderate compared with a naive exhaustive fine-scale search. By leveraging coarse tokens for efficient anatomical routing, CS3F-MS can achieve fine-grained lesion sensitivity that is unattainable with CS3F-C, without the prohibitive quadratic cost of full fine-scale matching. In practice, CS3F-C serves as the efficient default configuration, particularly under tight throughput or memory constraints, while CS3F-MS is recommended when the additional focal-lesion sensitivity justifies the increase in computational overhead.

\begin{table}[t]
	\centering
	\caption{Computational overhead of multi-scale scoring for a batch of
		180 volumes. Tokenization includes batch preparation for 3D
		processing, 2D encoder inference, and
		volumetric token construction. Scoring refers exclusively to the cross-subject scoring stage. Peak VRAM denotes maximum reserved GPU
		memory.}
	\label{tab:multiscale_efficiency}
	\scriptsize
	\setlength{\tabcolsep}{3.2pt}
	\begin{tabular}{lccccc}
		\toprule
		Method & Total & s/vol. & Tokenization & Scoring & Peak VRAM \\
		\midrule
		CS3F-C  & 575.8 & 3.2 & 486.0 & 60.7  & 5.5 GB \\
		CS3F-MS & 926.1 & 5.2 & 501.1 & 398.9 & 8.4 GB \\
		\bottomrule
	\end{tabular}
\end{table}

\subsection{Cross-Organ Validation on Lung CT}
\label{sec:lung_ct_validation}

To examine whether CS3F generalizes beyond brain MRI, we evaluate it on
lung CT. This setting differs from the preceding experiments in several
important aspects: the target organ is different, the imaging modality is CT
rather than MRI, the volumes are not standardized by atlas registration, and
the processed inputs are non-cubic. Lung CT therefore provides a meaningful
test of whether axis-wise volumetric scoring and cross-subject matching remain
effective outside atlas-aligned brain MRI. For the baselines, AnomalyCLIP and APRIL-GAN are applied slice-wise across all three anatomical axes, separately to the left and right segmented lung regions using the same masks and partitioning as CS3F.

As reported in Table~\ref{tab:lung_ct_results} and illustrated
qualitatively in Fig.~\ref{fig:lung_ct_qualitative}, CS3F operates
effectively in this setting without structural modifications. Consistent
with the brain MRI experiments, CS3F outperforms the CLIP-based zero-shot
baselines on lung CT in both patient-level detection and voxel-level
localization metrics. Among the variants, the efficient CS3F-C provides the
strongest patient-level detection and voxel-level AP, while the multi-scale
CS3F-MS achieves the best segmentation overlap, with a Dice score of
$33.0\%$. Additionally, we evaluate patient-wise lesion localization by
checking whether the thresholded anomaly map overlaps with at least one
annotated tumor voxel, following the criterion used in~\cite{kim20243d}. At $99\%$ TNR, CS3F-F successfully detects cancer regions in
$56$ out of $62$ patients ($90.3\%$), while CS3F-MS and CS3F-C succeed in $53$
($85.5\%$) and $49$ ($79.0\%$) cases, respectively. For reference, VMPR-UAD~\citep{kim20243d},
a specialized method for lung CT trained on 4,000 internal healthy CT samples,
reports approximately $93\%$ localization on MSD Lung. Although this is not a
direct comparison, this reference point shows that the zero-shot CS3F
framework achieves meaningful lung cancer localization in lung CT without
lung CT-specific training.

\begin{table}[t]
	\centering
	\caption{Cross-organ validation on lung CT. Values are reported as percentages.  \textbf{Bold} indicates the best zero-shot result.}
	\label{tab:lung_ct_results}
	\scriptsize
	\setlength{\tabcolsep}{2.8pt}
	\begin{tabular}{lccccccc}
		\toprule
		\multirow{2}{*}{Method}
		& \multicolumn{3}{c}{Patient-level}
		& \multicolumn{4}{c}{Voxel-level} \\
		\cmidrule(lr){2-4}
		\cmidrule(lr){5-8}
		& AUROC & AP & F1
		& AUROC & AP & Dice & IoU \\
		\midrule
		APRIL-GAN   & 61.5 & 64.7 & 70.0 & 90.1 & 2.8  & 9.7  & 5.9  \\
		AnomalyCLIP & 73.6 & 73.3 & 73.8 & 91.6 & 7.8  & 9.7  & 5.8  \\
		\hdashline \\ [-1ex]
		CS3F-F      & 78.4 & {73.9} & 80.0 & \textbf{99.0} & {32.4} & 32.9 & 23.3 \\
		CS3F-C      & \textbf{80.6} & \textbf{74.7} & \textbf{81.3} & 98.9 & \textbf{35.3} & 31.7 & 22.6 \\
		CS3F-MS     & 80.2 & \textbf{74.7} & {81.1} & \textbf{99.0} & 34.5 & \textbf{33.0} & \textbf{23.8} \\
		\bottomrule
	\end{tabular}
\end{table}

Overall, the lung CT results confirm the cross-organ applicability of CS3F.
The method remains functional under a different organ, a different imaging
modality, non-atlas-based preprocessing, and non-cubic 3D inputs --- all
conditions handled natively by the axis-wise scoring design without
modification.

\section{Ablation Study}
\subsection{Robustness to Batch Size}
CS3F estimates anomaly scores from cross-subject statistics; thus, its reliability depends on the number of test cases processed together. To study this, we divide the BraTS-METS T1w test set (180 volumes) into non-overlapping equal-sized chunks and process each chunk independently. In this way, anomaly scores are computed only within each chunk. This setup also mirrors a practical clinical workflow, where cases arrive and are analyzed periodically (e.g., over daily or weekly acquisition windows) rather than being accumulated into a single large cohort.

Table~\ref{tab:ablation_batch_chunking} shows that performance degrades smoothly as batch size decreases. From 2 to 6 chunks (corresponding to batch sizes of 90 down to 30 volumes), CS3F-C remains reliable, with patient-level AUROC above 95.9\%, voxel AP above 34.3\%, and Dice above 25.1\%. These results indicate that batches of 30 or more cases are minimal for any practical purposes. In the extreme case of 15 samples per batch, patient-level AUROC and Dice remain relatively stable, but voxel AP drops sharply from 34.3\% to 21.3\%. This suggests that voxel-wise ranking is more sensitive to the reduced diversity of reference samples. Overall, the results demonstrate that while larger batches are beneficial, reliable operation does not require processing the entire cohort at once.

\begin{table}[t]
	\centering
	\scriptsize
	\caption{Robustness to batch size on BraTS-METS T1w. The test set of 180 volumes is divided into non-overlapping chunks, and each chunk is processed independently. Results are reported as mean $\pm$ standard deviation over five random partitions.}
	\label{tab:ablation_batch_chunking}
	\begin{tabular}{ccccc}
		\toprule
		{\# Chunks} & {Batch size} & {P-AUROC} & {V-AP} & Dice \\
		\midrule
		1 & 180 & 97.0 & 38.4 &28.7 \\
		2  & 90 & $96.9 \pm 0.2$ & $38.3 \pm 1.4$ & $28.2 \pm 0.4$ \\
		4  & 45 & $96.8 \pm 0.4$ & $37.2 \pm 2.1$ & $27.9 \pm 0.8$ \\
		6  & 30 & $95.9 \pm 0.4$ & $34.3 \pm 2.7$ & $25.1 \pm 1.3$ \\
		12 & 15 & $95.4 \pm 0.8$ & $21.3 \pm 1.2$ & $23.4 \pm 1.5$ \\
		\bottomrule
	\end{tabular}
\end{table}

\subsection{Effect of Multi-Axis Score Fusion}

To assess the role of multi-axis context, we compare CS3F-C using individual anatomical axes and their combinations on BraTS-METS T1w. As shown in Table~\ref{tab:ablation_multiaxis}, the main advantage of using multiple axes is observed at voxel-level segmentation. Combining multiple axes recovers complementary spatial information and improves voxel-level AP and Dice over single-axis variants. For example, axial and coronal scoring achieve similar voxel AP values of $28.4\%$ and $28.0\%$, respectively, whereas their fusion increases voxel AP to $37.3\%$. This near $9$ percentage-point improvement indicates that different anatomical planes capture partially distinct anomalous signals that may not be visible or well localized from a single orientation. The full tri-axial configuration achieves the best voxel AP and Dice scores, reaching $38.4\%$ and $28.7\%$, respectively. These results support multi-axis score fusion as a simple and effective design for improving spatial localization.

\begin{table}[t]
	\centering
	\scriptsize
	\caption{Ablation of multi-axis score fusion for CS3F-C on BraTS-METS T1w. P-AUROC denotes patient-level AUROC, and V-AP denotes voxel-level average precision. Values are reported as percentages.}
	\label{tab:ablation_multiaxis}
	\begin{tabular}{lccc}
		\toprule
		{Viewpoint} & {P-AUROC} & {V-AP} & Dice \\
		\midrule
		A       & {98.7} & 28.4 & 23.5 \\
		C       & 94.1 & 28.0 & 23.7 \\
		S       & 93.6 & 18.8 & 20.5 \\
		A + S   & 97.7 & 32.6 & 27.2 \\
		A + C   & 97.7 & 37.3 & 27.5 \\
		C + S   & 94.2 & 31.0 & 25.9 \\
		A + C + S & 97.0 & {38.4} & {28.7} \\
		\bottomrule
	\end{tabular}
\end{table}

\subsection{Fusion Stage: Feature-Level Fusion versus Score-Level Fusion}
\label{sec:ablation_fusion_stage}
We compare score-level fusion with the feature-level fusion strategy used in \cite{gia2026trainingfree}. Feature-level fusion performs cross-subject matching after concatenating tokens from the axial, coronal, and sagittal views. This early fusion can reduce noisy matches in normal regions because each nearest-neighbor match must be jointly consistent across three anatomical axes. This explains its slightly stronger voxel-level performance on BraTS-METS T1w, as shown in Table~\ref{tab:ablation_score_fusion}.

Nevertheless, feature-level fusion has two important practical limitations. It increases peak memory usage from \(5.4\) GB to \(9.2\) GB and requires compatible token grids across the three axes, restricting the design to cubic volumes. Score-level fusion protocol in CS3F avoids these constraints by scoring each axis independently and fusing the resulting anomaly maps in voxel space, as illustrated in Figure~\ref{fig:cs3f_overall_pipeline}. Therefore, score-level fusion is adopted as the default design because it is more memory-efficient and supports non-cubic volumes and multi-scale C2F scoring.
\begin{table}[t]
	\centering
	\scriptsize
	\caption{Comparison between early fusion and score fusion for CS3F-C on BraTS-METS T1w. P-AUROC denotes patient-level AUROC, and V-AP denotes voxel-level average precision.}
	\label{tab:ablation_score_fusion}
	\begin{tabular}{lccccc}
		\toprule
		{Fusion design} & {P-AUROC} & {V-AP} & {Dice} & {VRAM} & {Time/vol.} \\
		\midrule
		Early fusion & {97.3} & {40.8} & {30.0} & 9.2 GB & {3.1 s} \\
		Score fusion & 97.0 & 38.4 & 28.7 & {5.4 GB} & 3.2 s \\
		\bottomrule
	\end{tabular}
\end{table}

\subsection{Sensitivity to the projection dimension}
\label{sec:projection_dim}
\begin{figure}[!t]
	\centering
	\includegraphics[width=\linewidth]{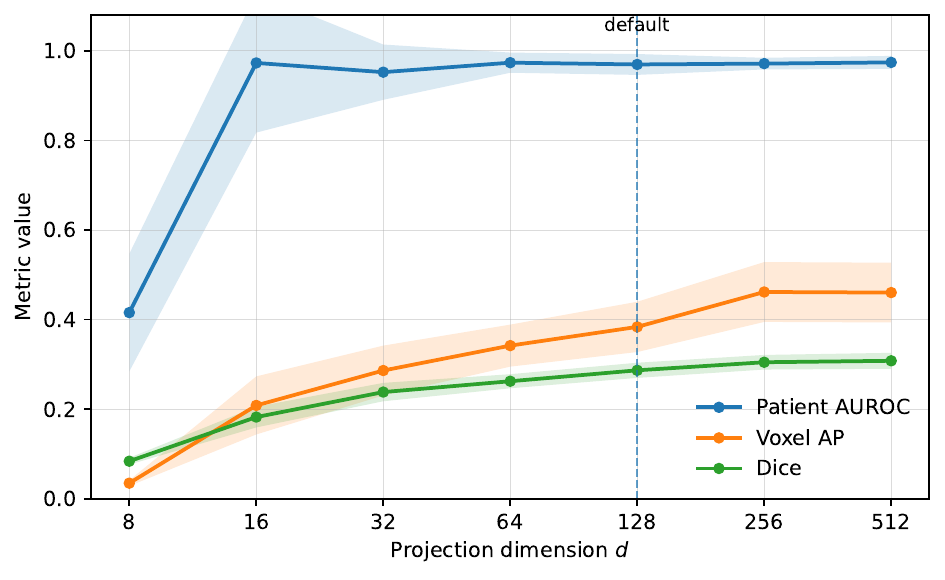}
	\caption{
		Effect of the random projection dimension \(d\). Shaded regions denote
		sample standard deviation across 5 random seeds.
	}
	\label{fig:ablation_projection_dim}
\end{figure}
We ablate the random projection dimension \(d\) in
Fig.~\ref{fig:ablation_projection_dim}. Very small dimensions (e.g., $d=8$) substantially degrade both patient-level detection and voxel-level
anomaly segmentation, suggesting that excessive compression can distort anomaly
signals and pairwise token distances. Performance becomes stable once
\(d \geq 64\): patient-level AUROC saturates early, while voxel-level AP and Dice
continue to improve moderately with larger \(d\). In practice, \(d\) should be chosen as large as the available hardware
allows. Nevertheless, moderate
dimensions such as \(d=128\) remain strong efficiency-oriented choices,
providing stable performance while reducing the memory and computational
cost of pairwise matching.

\subsection{Sensitivity to the C2F Routing Budget}
\label{sec:ablation_c2f_L}

We analyze the sensitivity of  the multi-scale CS3F-MS to the C2F routing budget \(L\), which controls the number of coarse neighbors used for restricted fine-scale matching. As shown in Table~\ref{tab:ablation_c2f_L}, increasing \(L\) only slightly improves V-AP and Dice, while the processing time increases rapidly. In particular, using a very large budget such as \(L=256\) substantially increases the time per volume but provides only marginal gains compared to moderate values of \(L\). The lesion-level TPR at \(99\%\) TNR also does not improve with larger \(L\), indicating that larger routing budgets are not cost-effective for high-specificity lesion detection. Overall, moderate values of \(L\) are sufficient for C2F.
\begin{table}[t]
	\centering
	\caption{
		Sensitivity of CS3F-MS to the C2F routing budget \(L\) on BraTS-METS T1w.
		All metrics are reported as percentages except time per sample and FP/V.
		FP/V denotes the number of false positive lesions per volume at \(99\%\) TNR.
	}
	\scriptsize
	\label{tab:ablation_c2f_L}
	\begin{tabular}{ccccccc}
		\toprule
		\(L\) 
		& {s/vol.}
		& {P-AUROC}
		& {V-AP}
		& Dice
		& LTPR\(_{99}\)
		& {FP/V} \\
		\midrule
		1   & 3.4 s  & 97.2 & 31.9 & 25.8 & 41.0 & 3.4 \\
		4   & 3.7 s  & 97.2 & 32.7 & 26.1 & 39.0 & 3.2 \\
		16  & 5.2 s  & 97.0 & 33.1 & 26.2 & 38.5 & 3.3 \\
		64  & 7.4 s  & 96.9 & 33.3 & 26.3 & 38.0 & 3.4 \\
		256 & 21.3 s & 96.8 & 33.5 & 26.3 & 38.0 & 3.4 \\
		\bottomrule
	\end{tabular}
\end{table}

\section{Discussion and conclusion}

This work proposes CS3F, a batch-based framework for zero-shot anomaly detection in 3D medical volumes. CS3F combines frozen 2D foundation-model features, multi-axis volumetric tokenization, and axis-specific cross-subject mutual similarity scoring to localize regions whose representations are poorly matched by other subjects in the same batch. The proposed coarse-to-fine routing further allows this framework to reach finer-resolution scoring without the prohibitive cost of exhaustive matching. The framework requires no lesion annotations, no clean normal training datasets, and no pathology-specific adaptation, yet achieves competitive performance across brain MRI benchmarks covering metastases, glioma, and stroke, as well as  lung CT.

This paper also illustrates both the plausibility and limitations of using 2D foundation models for volumetric anomaly localization without fine-tuning or task-specific adaptation. On the positive side, the results show that local features extracted from generic 2D foundation models can retain sufficient anatomical and pathological information to support anomaly scoring in 3D volumes when paired with a suitable pipeline. However, this reliance on unadapted representations also constrains the method. Lesions can only be detected when they induce distinguishable anomaly signals in the extracted features. Small, subtle, or low-contrast abnormalities may remain entangled with normal tissue representations, limiting localization performance even with finer tokenization.

CS3F also inherits limitations from its batch-based cross-subject matching setting. The Doppelg\"{a}nger assumption~\ref{assump:doppelganger} assumes that the processed batch contains sufficient recurrent normal structure for mutual similarity scoring to be meaningful. The method is therefore most naturally suited to cohort-level screening, large-scale dataset curation, and research pipelines, rather than immediate single-subject inference. Moreover, cross-subject matching requires a certain level of spatial or region-level standardization. In this work, brain MRI volumes are registered to a common space, whereas lung CT volumes are analyzed within segmented lung regions. Although the lung CT experiment suggests that CS3F can operate under less standardized preprocessing than brain MRI, its extension to organs with weaker spatial correspondence, larger deformation, or less reliable segmentation remains to be investigated.

Despite these limitations, CS3F represents a practical step toward zero-shot anomaly detection in 3D medical volumes. By eliminating dependence on separately curated healthy-control datasets and manual lesion annotations, the framework offers a practical pathway for cohort-level screening, large-scale dataset curation, and automated quality control—use cases where processing batches of studies is clinically natural. Nevertheless, current performance, particularly for small and low-contrast lesions, has not yet reached the level required for prospective clinical decision support. In future work, we aim to investigate lightweight adaptation strategies for foundation models on auxiliary medical data to improve lesion–normal separability while preserving the zero-shot character of the target task; and to explore multi-modal and multi-sequence fusion that can integrate complementary anomaly evidence across imaging contrasts.

\clearpage
\newcommand{\SegmentationGrid}[1]{%
	\begin{tikzpicture}
		\node[anchor=south west, inner sep=0] (fig) at (0,0)
		{\includegraphics[width=0.92\linewidth]{#1}};
		
		\foreach \x/\name in {
			0.0625/GT,
			0.1875/DAE,
			0.3125/pDDPM,
			0.4375/APRIL-GAN,
			0.5625/AnomalyCLIP,
			0.6875/CS3F-F,
			0.8125/CS3F-C,
			0.9375/CS3F-MS
		} {
			\node[anchor=south, font=\footnotesize\bfseries]
			at ($($(fig.north west)!\x!(fig.north east)$)+(0,2pt)$) {\name};
		}
		
		\foreach \y/\name in {
			0.9333/Subject 1,
			0.6000/Subject 2,
			0.2667/Subject 3
		} {
			\node[anchor=east, rotate=90, font=\footnotesize\bfseries]
			at ($($(fig.south west)!\y!(fig.north west)$)+(-8pt,0)$) {\name};
		}
	\end{tikzpicture}%
}

\begin{figure*}[!t]
	\centering
	
	\begin{subfigure}{\textwidth}
		\centering
		\SegmentationGrid{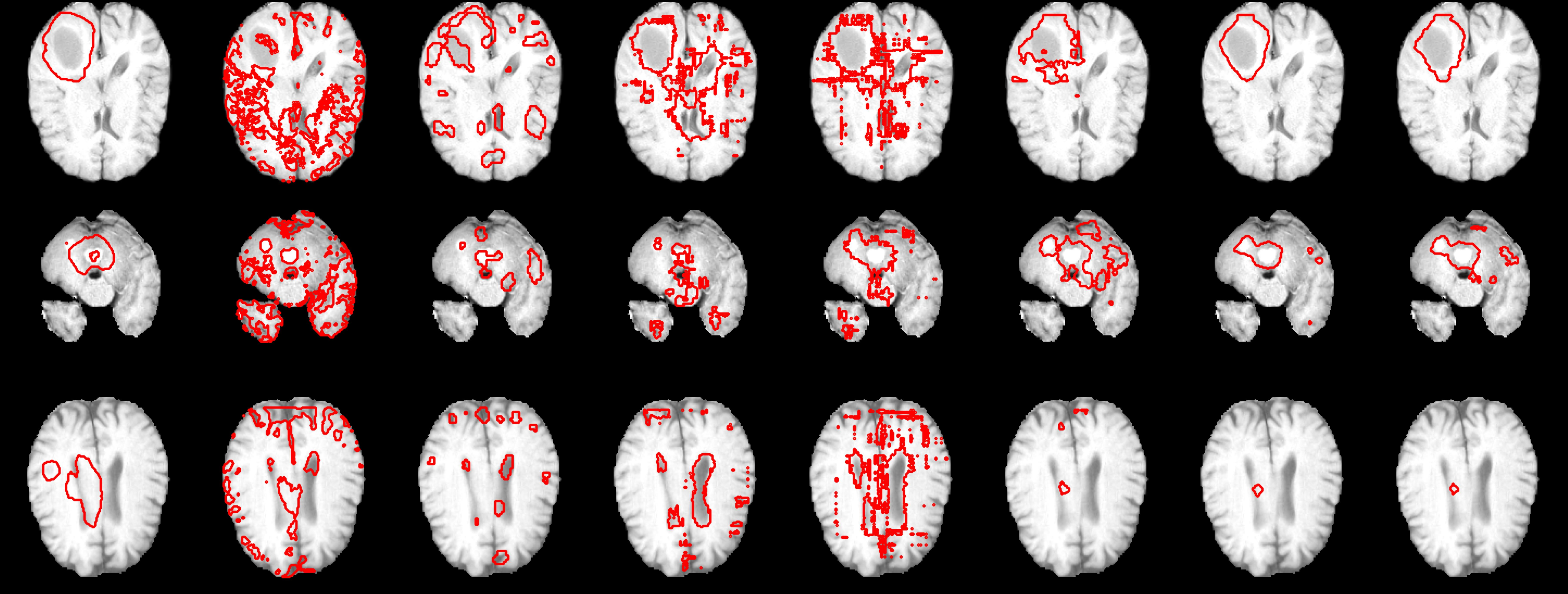}
		\caption{BraTS-METS T1w.}
		\label{fig:mets_t1_qualitative_sub}
	\end{subfigure}
	
	\vspace{0.6em}
	
	\begin{subfigure}{\textwidth}
		\centering
		\SegmentationGrid{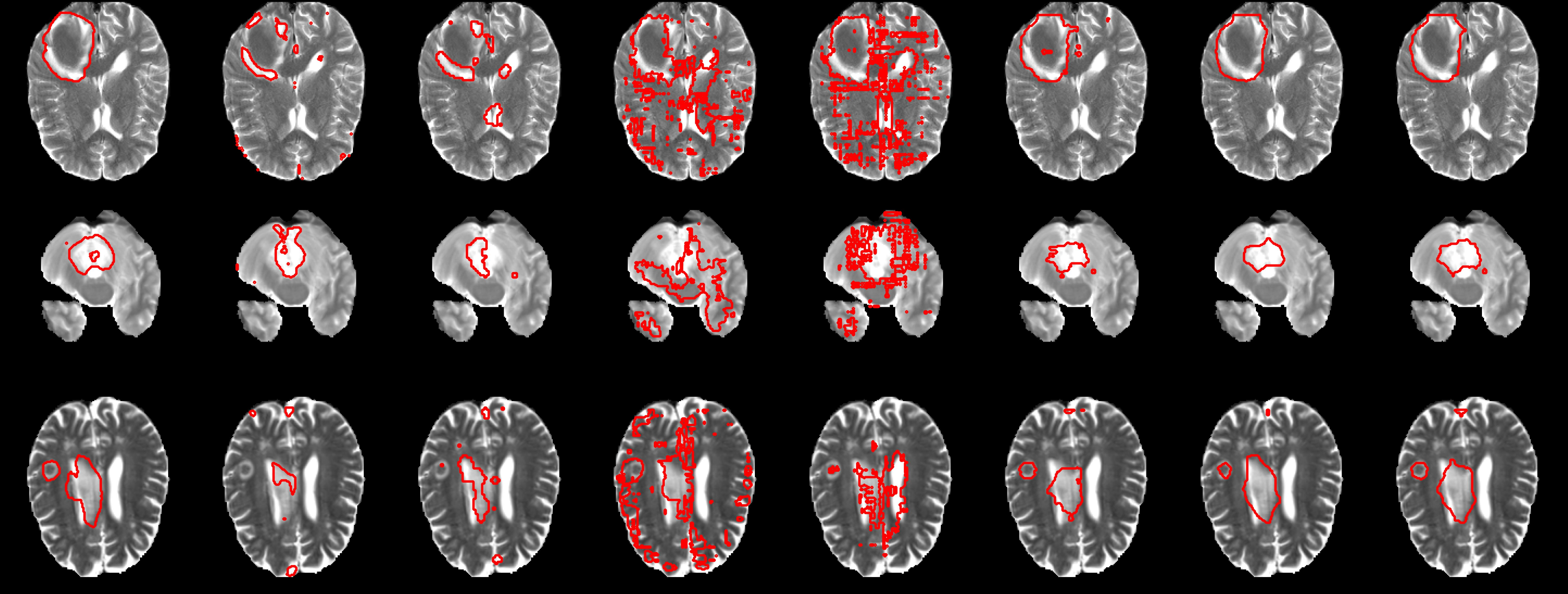}
		\caption{BraTS-METS T2w.}
		\label{fig:mets_t2_qualitative_sub}
	\end{subfigure}
	\caption{
		Qualitative anomaly segmentation results on BraTS-METS using T1w and T2w inputs.
		The same subjects and axial slices are shown for both modalities to enable direct
		comparison of modality-dependent segmentation behavior. Red contours indicate the
		boundaries of the predicted anomaly masks at maximum Dice operating point, overlaid on the original images.
	}
	\label{fig:mets_t1_t2_qualitative}
\end{figure*}

\begin{figure*}[!t]
	\centering
	\begin{tikzpicture}
		\node[anchor=south west, inner sep=0] (fig) at (0,0)
		{\includegraphics[width=0.92\linewidth]{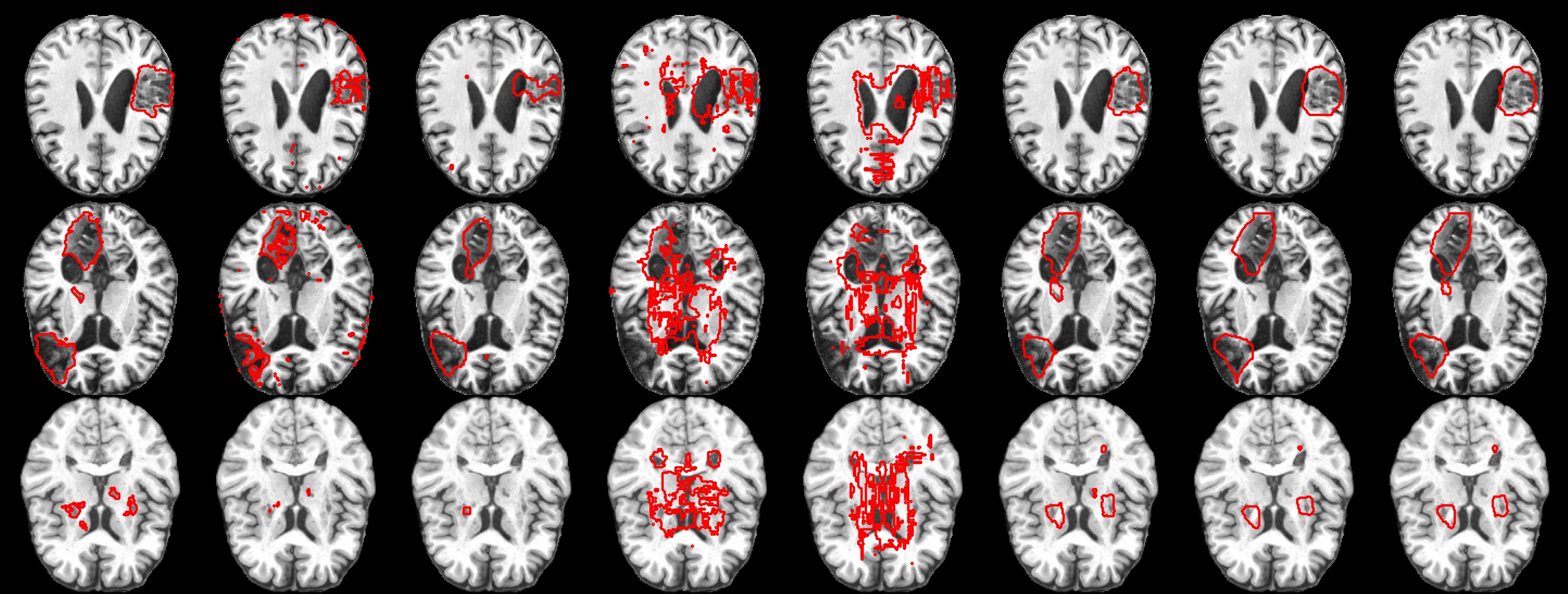}};
		
		\foreach \x/\name in {
			0.0625/GT,
			0.1875/DAE,
			0.3125/pDDPM,
			0.4375/APRIL-GAN,
			0.5625/AnomalyCLIP,
			0.6875/CS3F-F,
			0.8125/CS3F-C,
			0.9375/CS3F-MS
		} {
			\node[anchor=south, font=\footnotesize\bfseries]
			at ($($(fig.north west)!\x!(fig.north east)$)+(0,2pt)$) {\name};
		}
		
		\foreach \y/\name in {
			0.9333/Subject 1,
			0.6000/Subject 2,
			0.2667/Subject 3
		} {
			\node[anchor=east, rotate=90, font=\footnotesize\bfseries]
			at ($($(fig.south west)!\y!(fig.north west)$)+(-8pt,0)$) {\name};
		}
	\end{tikzpicture}
	\caption{ Qualitative anomaly segmentation results on T1w inputs from ATLAS. Red contours indicate the boundaries of the predicted anomaly masks at maximum Dice operating point, overlaid on the original images. The examples illustrate stroke lesion localization across subjects with different lesion extent and appearance. }
	\label{fig:atlas_qualitative}
\end{figure*}

\begin{figure*}[!t]
	\centering
	\begin{tikzpicture}
		\node[anchor=south west, inner sep=0] (fig) at (0,0)
		{\includegraphics[width=0.92\linewidth]{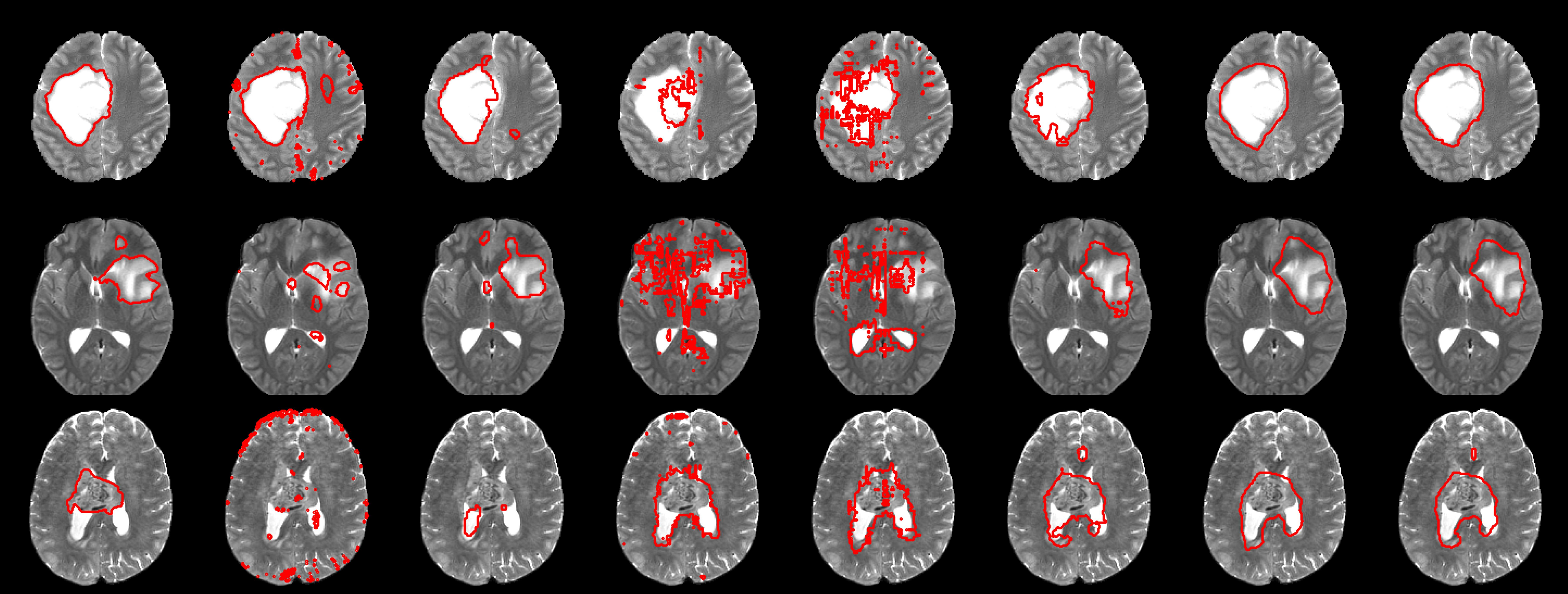}};
		
		\foreach \x/\name in {
			0.0625/GT,
			0.1875/DAE,
			0.3125/pDDPM,
			0.4375/APRIL-GAN,
			0.5625/AnomalyCLIP,
			0.6875/CS3F-F,
			0.8125/CS3F-C,
			0.9375/CS3F-MS
		} {
			\node[anchor=south, font=\footnotesize\bfseries]
			at ($($(fig.north west)!\x!(fig.north east)$)+(0,2pt)$) {\name};
		}
		
		\foreach \y/\name in {
			0.9333/Subject 1,
			0.6000/Subject 2,
			0.2667/Subject 3
		} {
			\node[anchor=east, rotate=90, font=\footnotesize\bfseries]
			at ($($(fig.south west)!\y!(fig.north west)$)+(-8pt,0)$) {\name};
		}
	\end{tikzpicture}
	\caption{ Qualitative anomaly segmentation results on T2w inputs from BraTS-GLI. Red contours indicate the boundaries of the predicted anomaly masks at maximum Dice operating point, overlaid on the original images. The examples illustrate glioma localization across subjects with heterogeneous tumor appearance and spatial extent. }
	\label{fig:gli_qualitative}
\end{figure*}

\begin{figure*}[!t]
	\centering
	\begin{tikzpicture}
		\node[anchor=south west, inner sep=0] (img) at (0,0)
		{\includegraphics[width=\textwidth]{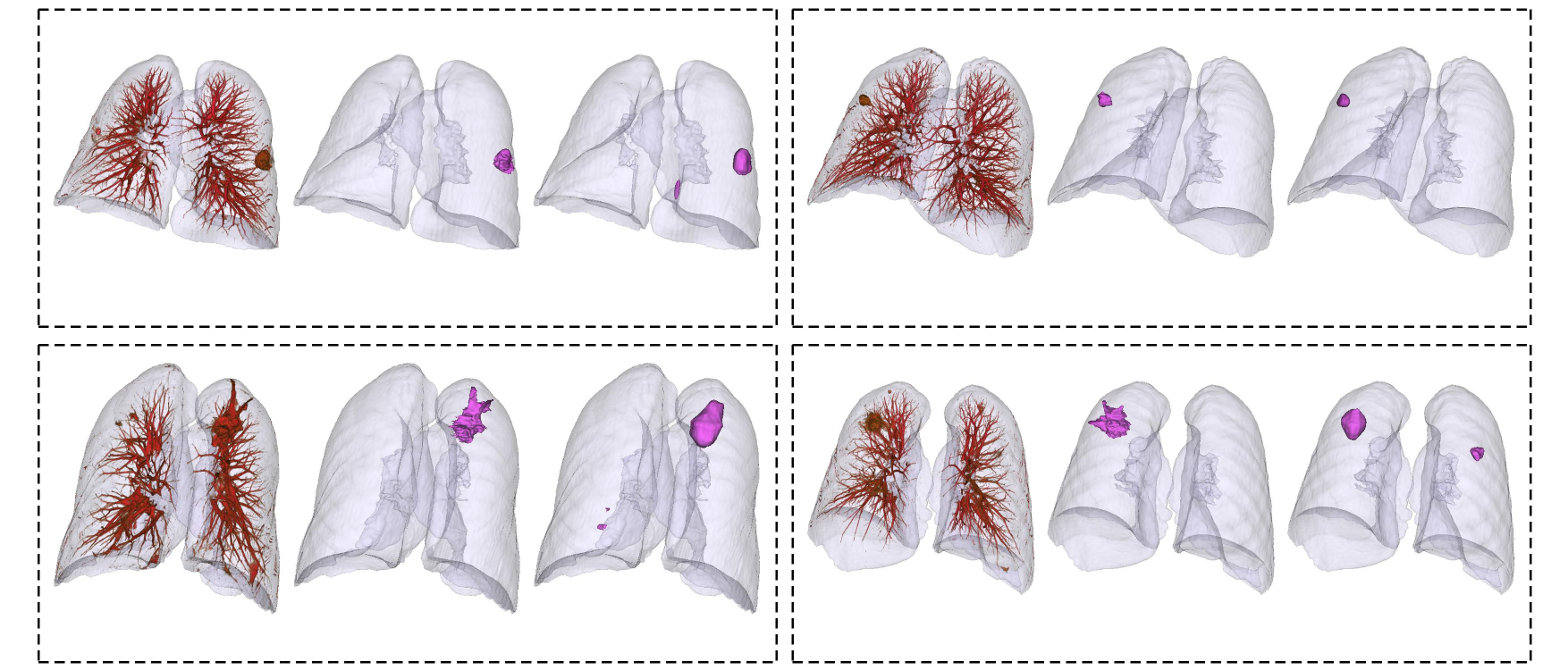}};
		
		\begin{scope}[x={(img.south east)}, y={(img.north west)}]
			
			\node[font=\small\bfseries] at (0.115,0.565) {Segmented lung};
			\node[font=\small\bfseries] at (0.265,0.565) {Ground truth};
			\node[font=\small\bfseries] at (0.415,0.565) {Prediction};
			
			\node[font=\small\bfseries] at (0.595,0.565) {Segmented lung};
			\node[font=\small\bfseries] at (0.745,0.565) {Ground truth};
			\node[font=\small\bfseries] at (0.895,0.565) {Prediction};
			
			\node[font=\small\bfseries] at (0.115,0.075) {Segmented lung};
			\node[font=\small\bfseries] at (0.265,0.075) {Ground truth};
			\node[font=\small\bfseries] at (0.415,0.075) {Prediction};
			
			\node[font=\small\bfseries] at (0.595,0.075) {Segmented lung};
			\node[font=\small\bfseries] at (0.745,0.075) {Ground truth};
			\node[font=\small\bfseries] at (0.895,0.075) {Prediction};
			
		\end{scope}
	\end{tikzpicture}
	\caption{Examples of 3D anomaly segmentation maps generated by CS3F-C.}
	\label{fig:lung_ct_qualitative}
\end{figure*}

\clearpage


	\section*{Declaration of competing interest}
	
	The authors declare that they have no known competing financial interests or personal relationships that could have appeared to influence the work reported in this paper.
	
\section*{Data availability}

All datasets used in this study are publicly available and de-identified. The study used only retrospective public data and did not involve new human-subject data collection.

	\section*{Code availability}
	
	The source code will be made publicly available at \url{https://github.com/DumBringer/CS3F}.
	
	\section*{Acknowledgement}
	The author thanks Prof. Jaehyun Ahn for supervision and academic guidance during the preliminary stage of this work.
	\appendix
	
\section{Classification difficulty of small and medium lesions with generic foundation models}
\label{app:lesion_separability}

This appendix examines whether the weak detection of small and medium
lesions is related to the classification difficulty between lesion
representations and normal-tissue representations. For modality comparison,
we evaluate BraTS-METS T1w, T2w, and  post-contrast T1c. For pathology comparison, we
evaluate ATLAS T1w. We use only the 65 abnormal test cases
from the main experiment and do not include IXI normal samples, since IXI
does not provide T1c scans. CS3F-F is also evaluated on the same 65 abnormal
cases. All features are extracted using the default setting of the main
experiment.

Here, we analyze raw 2D patch tokens obtained by applying the frozen
foundation model (DINOv2) directly to individual axial slices. For each small or medium lesion, we select
one lesion token, defined as the raw patch token with the largest spatial
overlap with that lesion. This avoids artificially increasing similarity by
including multiple correlated tokens from the same lesion. Normal tokens are
selected only from regions containing \(100\%\) normal voxels. For each
abnormal volume, we randomly select five slices that are at least five
slices away from the selected lesion-token locations, and extract all fully
normal tokens from these slices.

Classification difficulty is quantified via the adapted leave-one-out error of a 1-Nearest Neighbor (1-NN) classifier, known as the adapted \(N3\) metric \citep{barella2021assessing}. This metric is well-suited to highly imbalanced settings such as ours. Adapted \(N3\) is computed for the lesion class and represents the fraction of small- and medium-lesion tokens whose nearest neighbor belongs to the normal-token class. Thus, values close to 1 indicate strong local overlap (high difficulty) between lesion and normal-tissue representations, whereas lower values reflect greater distinguishability.

\begin{table}[t]
	\centering
	\scriptsize
	\setlength{\tabcolsep}{3pt}
	\caption{
		Adapted \(N3\) for small- and medium-lesion tokens and corresponding
		lesion-wise sensitivity at 99\% TNR of CS3F-F. Lower \(N3\) indicates easier
		lesion-normal separation. Values are reported as percentages.
	}
	\label{tab:n3_small_medium_lesions}
	\begin{tabular}{lcccc}
		\toprule
		{Metric}
		& {METS T1w}
		& {METS T2w}
		& {METS T1c}
		& {ATLAS T1w} \\
		\midrule
		\(N3\), \(\ell=6\)  & 98.4 & 98.4 & 95.1 & 89.5 \\
		\(N3\), \(\ell=12\) & 97.6 & 97.6 & 81.3 & 70.5 \\
		\(N3\), \(\ell=18\) & 96.7 & 100.0 & 74.8 & 66.7 \\
		\(N3\), \(\ell=24\) & 94.3 & 96.7 & 72.4 & 67.6 \\
		\midrule
		S-LTPR\(_{99}\)              & 5.2  & 1.7  & 17.2 & 15.4 \\
		M-LTPR\(_{99}\)              & 10.6 & 3.0  & 42.4 & 45.8 \\
		Dice & 19.6 & 37.6 & 31.4 & 25.2\\
		\bottomrule
	\end{tabular}
\end{table}

Table~\ref{tab:n3_small_medium_lesions} demonstrates that detection performance on small and medium lesions is strongly governed by the degree of overlap between lesion and normal-tissue tokens in the foundation model's feature space. When adapted \(N3\) remains close to 1, as observed in BraTS-METS T1w and T2w, lesion tokens are locally entangled with normal-tissue representations. Consequently, the anomaly signal is weak, yielding low S-LTPR and M-LTPR. By contrast, BraTS-METS T1c shows lower adapted \(N3\) in deeper layers and correspondingly higher lesion-wise sensitivity, consistent with the established clinical utility of contrast-enhanced T1-weighted imaging for detecting and characterizing small metastases \citep{fink2013imaging,flouri2025longitudinal}. ATLAS T1w confirms this pattern on small- and medium-lesion detection for a different pathology. Overall, these results confirm that small- and medium-lesion detection in CS3F is limited not only by token resolution but also by the foundation model's ability to encode a sufficiently discriminative lesion signal.

\section{Alternate 2D ViT Encoders}
\label{app:backbone_robustness}

We further evaluate CS3F on BraTS-METS T1w using three additional 2D ViT backbones, while keeping all settings the same as in the main experiments and changing only the feature extractor. Specifically, we consider DINOv3, a recent general-purpose visual foundation model \citep{dinov3}; ViT-L/14@224, the image encoder of CLIP and a widely used vision-language foundation model; and PMC-CLIP, a medical vision-language model pretrained on 1.6M biomedical image--text pairs \citep{lin2023pmc}. As shown in Tables~\ref{tab:backbone_main_results} and~\ref{tab:backbone_fixed_spec_results}, CS3F maintains effective patient-level detection and voxel-level localization across different backbone families. The results also suggest that medical-domain pretraining provides stronger local representations for lesion segmentation.

\begin{table}[t]
	\centering
	\caption{Performance of different ViT backbones on BraTS-METS T1w. Values are reported as percentages.}
	\label{tab:backbone_main_results}
	\scriptsize
	\setlength{\tabcolsep}{2.8pt}
	\begin{tabular}{llccccccc}
		\toprule
		\multirow{2}{*}{Backbone}
		& \multirow{2}{*}{Method}
		& \multicolumn{3}{c}{Patient-level}
		& \multicolumn{4}{c}{Voxel-level} \\
		\cmidrule(lr){3-5}
		\cmidrule(lr){6-9}
		& & AUROC & AP & F1
		& AUROC & AP & Dice & IoU \\
		\midrule
		\multirow{3}{*}{DINOv3}
		& CS3F-F  & 98.1 & {96.1} & 91.7 & 93.6 & 25.9 & 22.6 & 14.8 \\
		& CS3F-C  & 98.2 & 94.6 & {94.9} & {96.2} & {45.2} & {30.1} & {21.3} \\
		& CS3F-MS & {98.5} & 95.9 & 94.7 & 95.4 & 37.9 & 27.8 & 19.0 \\
		\midrule
		\multirow{3}{*}{ViT-L/14}
		& CS3F-F  & 96.4 & 89.8 & 89.7 & 89.3 & 17.4 & 14.9 & 9.0 \\
		& CS3F-C  & 97.4 & 92.5 & {93.3} & {93.6} & {27.6} & {20.7} & {13.3} \\
		& CS3F-MS & {97.5} & {92.7} & 91.4 & 92.2 & 23.7 & 18.4 & 11.6 \\
		\midrule
		\multirow{3}{*}{PMC-CLIP}
		& CS3F-F  & 96.3 & 91.2 & 90.1 & 96.1 & 46.6 & 30.7 & 21.4 \\
		& CS3F-C  & {97.6} & {92.8} & 93.2 & {97.1} & {57.6} & {34.8} & {25.4} \\
		& CS3F-MS & 97.4 & 92.5 & {94.0} & 96.8 & 54.3 & 33.6 & 24.3 \\
		\bottomrule
	\end{tabular}
\end{table}

\begin{table}[t]
	\centering
	\caption{Performance of different ViT backbones at $99\%$ specificity on BraTS-METS T1w. Values are reported as percentage, except FP/V.}
	\label{tab:backbone_fixed_spec_results}
	\scriptsize
	\setlength{\tabcolsep}{3.2pt}
	\begin{tabular}{llccccc}
		\toprule
		\multirow{2}{*}{Backbone}
		& \multirow{2}{*}{Method}
		& \multicolumn{3}{c}{Lesion-level}
		& \multicolumn{2}{c}{Voxel-level} \\
		\cmidrule(lr){3-5}
		\cmidrule(lr){6-7}
		& & S-LTPR & M-LTPR & L-LTPR
		& Dice & Avg FP/V \\
		\midrule
		\multirow{3}{*}{DINOv3}
		& CS3F-F  & {19.0} & {24.2} & {79.0} & 21.2 & 8.5 \\
		& CS3F-C  & 10.3 & 12.1 & 67.1 & {29.6} & {1.3} \\
		& CS3F-MS & 13.8 & 15.2 & 68.4 & 26.5 & 2.6 \\
		\midrule
		\multirow{3}{*}{ViT-L/14}
		& CS3F-F  & {15.5} & {28.8} & {84.2} & 14.6 & 13.9 \\
		& CS3F-C  & 6.9 & 9.1 & 73.7 & {19.9} & {3.6} \\
		& CS3F-MS & 8.6 & 16.7 & 79.0 & 17.9 & 5.9 \\
		\midrule
		\multirow{3}{*}{PMC-CLIP}
		& CS3F-F  & {17.2} & {28.8} & {82.9} & 30.7 & 9.4 \\
		& CS3F-C  & 6.9 & 18.2 & 79.0 & {34.5} & {2.8} \\
		& CS3F-MS & 8.6 & 21.2 & 81.6 & 33.5 & 4.0 \\
		\bottomrule
	\end{tabular}
\end{table}

	\bibliographystyle{elsarticle-harv}
	\bibliography{references}
	
\end{document}

%% file: overall_pipeline.tex
\begin{figure*}[t]
	\centering
	
	\definecolor{axialcolor}{RGB}{0,135,95} 
	\definecolor{sagittalcolor}{RGB}{25,85,165}
	\definecolor{coronalcolor}{RGB}{175,78,0}
	\definecolor{dotcolor}{RGB}{100,100,100}
	
	\begin{tikzpicture}[
		x=1cm,
		y=1cm,
		line join=round,
		line cap=round,
		>=Latex,
		font=\sffamily,
		timeline/.style={draw=gray!55, line width=0.9pt},
		timeline merge/.style={draw=gray!55, line width=0.85pt},
		timeline dot/.style={circle, draw=gray!55, fill=gray!55, inner sep=0pt, minimum size=4.1pt},
		axis name/.style={font=\scriptsize, text=gray!40!black, anchor=east},
		step label/.style={font=\scriptsize, text=gray!30!black, align=center, anchor=north},
		pipe arrow/.style={-{Latex[length=2.1mm]}, draw=gray!45!black, line width=0.85pt},
		score arrow/.style={-{Latex[length=2.0mm]}, draw=red!70!black!65, line width=0.8pt},
		mutual score arrow/.style={{Latex[length=2.0mm]}-{Latex[length=2.0mm]}, draw=blue!70!black!65, line width=0.8pt},
		token/.style={draw=blue!80!black!70, fill=blue!80!black!70, line width=0.45pt, minimum width=0.16cm, minimum height=0.12cm, inner sep=0pt},
		tokenA/.style={token, left color=orange!85!yellow, right color=red!80!black, draw=red!70!black!70},
		tokenB/.style={token, top color=yellow!90, bottom color=orange!85!red, draw=orange!75!black},
		tokenC/.style={token, left color=cyan!65, right color=blue!80!black!70, draw=cyan!55!black},
		tokenD/.style={token, top color=red!55!orange, bottom color=red!85!black, draw=red!80!black},
		bracket/.style={draw=gray!45!black, line width=0.5pt},
		volume edge/.style={draw=gray!35!black, line width=0.55pt},
		coarse/.style={draw=blue!80!black!70, fill=blue!35, fill opacity=0.38, line width=0.75pt},
		fine/.style={draw=red!70!black!60, fill=orange!45!red!35, fill opacity=0.78, line width=0.65pt},
		score box/.style={draw=gray!50!black, fill=white, rounded corners=1.5pt, line width=0.6pt, inner sep=2pt, font=\scriptsize, text=gray!25!black},
		step boundary/.style={draw=gray!45, line width=0.55pt, densely dashed},
		step dot/.style={
			circle,
			fill=dotcolor,
			draw=dotcolor,
			inner sep=0pt,
			minimum size=4.4pt
		},
		]
		
		\newcommand{\BrainImageOne}[3]{%
			\node[inner sep=0pt, opacity=0.6] at (#1,#2)
			{\includegraphics[width=#3]{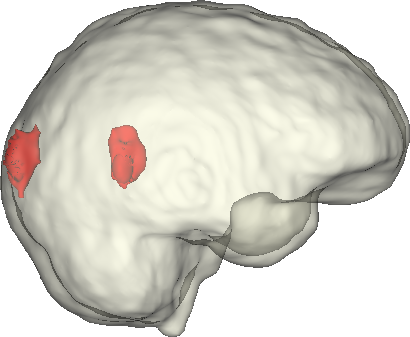}};
		}
		
		\newcommand{\BrainImageTwo}[3]{%
			\node[inner sep=0pt, opacity=0.6] at (#1,#2)
			{\includegraphics[width=#3]{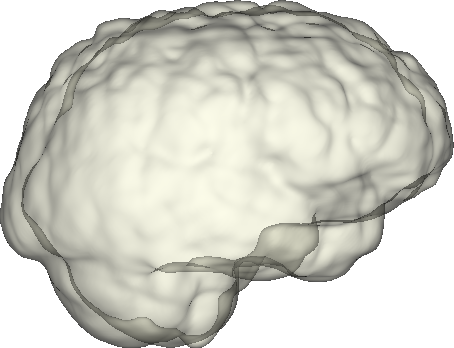}};
		}
		
		\newcommand{\BrainImageThree}[3]{%
			\node[inner sep=0pt, opacity=0.6] at (#1,#2)
			{\includegraphics[width=#3]{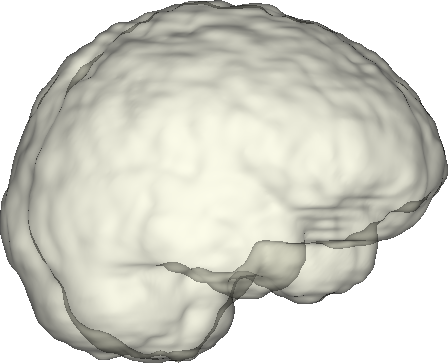}};
		}
		
		\newcommand{\TokenList}[3]{%
			\begin{scope}[shift={(#1,#2)}, scale=#3]
				\draw[bracket] (-0.15,0.16) -- (-0.15,0.22) -- (0.15,0.22) -- (0.15,0.16);
				\draw[bracket] (-0.15,-.88) -- (-0.15,-0.94) -- (0.15,-0.94) -- (0.15,-0.88);
				\node[token] at (0,0.08) {};
				\node[token] at (0,-0.20) {};
				\node[font=\scriptsize, inner sep=0pt, text=gray!45!black] at (0,-0.48) {$\vdots$};
				\node[token] at (0,-0.78) {};
			\end{scope}
		}
		
		\newcommand{\TokenListAlt}[3]{%
			\begin{scope}[shift={(#1,#2)}, scale=#3]
				\draw[bracket] (-0.15,0.16) -- (-0.15,0.22) -- (0.15,0.22) -- (0.15,0.16);
				\draw[bracket] (-0.15,-.88) -- (-0.15,-0.94) -- (0.15,-0.94) -- (0.15,-0.88);
				\node[token] at (0,0.08) {};
				\node[tokenB] at (0,-0.20) {};
				\node[font=\scriptsize, inner sep=0pt, text=gray!45!black] at (0,-0.48) {$\vdots$};
				\node[token] at (0,-0.78) {};
			\end{scope}
		}
		
		\newcommand{\DrawCube}[7]{%
			\begin{scope}[shift={(#1,#2)}, scale=#6]
				\coordinate (A) at (0,0);
				\coordinate (B) at (#3,0);
				\coordinate (C) at (#3,#4);
				\coordinate (D) at (0,#4);
				\coordinate (E) at (#5,#5*0.52);
				\coordinate (F) at (#3+#5,#5*0.52);
				\coordinate (G) at (#3+#5,#4+#5*0.52);
				\coordinate (H) at (#5,#4+#5*0.52);
				\path[fill=#7, draw=blue!80!black!70, line width=0.5pt, fill opacity=0.35]
				(A) -- (B) -- (C) -- (D) -- cycle;
				\path[fill=#7, draw=blue!80!black!70, line width=0.5pt, fill opacity=0.28]
				(D) -- (C) -- (G) -- (H) -- cycle;
				\path[fill=#7, draw=blue!80!black!70, line width=0.5pt, fill opacity=0.22]
				(B) -- (C) -- (G) -- (F) -- cycle;
				\draw[blue!80!black!70, line width=0.5pt]
				(E) -- (F) -- (G) -- (H) -- cycle
				(A) -- (B) -- (C) -- (D) -- cycle
				(A) -- (E)
				(B) -- (F)
				(C) -- (G)
				(D) -- (H);
			\end{scope}
		}
		
		\newcommand{\DrawWireCube}[6]{%
			\begin{scope}[shift={(#1,#2)}, scale=#6]
				\coordinate (A) at (0,0);
				\coordinate (B) at (#3,0);
				\coordinate (C) at (#3,#4);
				\coordinate (D) at (0,#4);
				\coordinate (E) at (#5,#5*0.52);
				\coordinate (F) at (#3+#5,#5*0.52);
				\coordinate (G) at (#3+#5,#4+#5*0.52);
				\coordinate (H) at (#5,#4+#5*0.52);
				\draw[black, line width=0.5pt]
				(E) -- (F) -- (G) -- (H) -- cycle
				(A) -- (B) -- (C) -- (D) -- cycle
				(A) -- (E)
				(B) -- (F)
				(C) -- (G)
				(D) -- (H);
			\end{scope}
		}
		
		\newcommand{\DrawFineSlab}[6]{%
			\begin{scope}[shift={(#1,#2)}, scale=#6]
				\coordinate (A) at (0,0);
				\coordinate (B) at (#3,0);
				\coordinate (C) at (#3,#4);
				\coordinate (D) at (0,#4);
				\coordinate (E) at (#5,#5*0.52);
				\coordinate (F) at (#3+#5,#5*0.52);
				\coordinate (G) at (#3+#5,#4+#5*0.52);
				\coordinate (H) at (#5,#4+#5*0.52);
				\path[fine] (A) -- (B) -- (C) -- (D) -- cycle;
				\path[fine] (D) -- (C) -- (G) -- (H) -- cycle;
				\path[fine] (B) -- (C) -- (G) -- (F) -- cycle;
				\draw[red!70!black!60, line width=0.4pt]
				(E) -- (F) -- (G) -- (H) -- cycle
				(A) -- (B) -- (C) -- (D) -- cycle
				(A) -- (E)
				(B) -- (F)
				(C) -- (G)
				(D) -- (H);
			\end{scope}
		}
		
		\newcommand{\TimelineRowFour}[8]{%
			\draw[draw=#8, line width=0.9pt] (#3,#1) -- (#6,#1);
			\foreach \xp/\idx in {#3/1,#4/2,#5/3,#6/4}{%
				\fill[#8] (\xp,#1) circle (2.5pt);
				\node[inner sep=0pt] (#7\idx) at (\xp,#1) {};
			}
			\node[axis name, text=#8!85!black] at ({#3-0.22},#1) {#2};
		}
		
		\coordinate (T1) at (0.8,0);
		\coordinate (T2) at (4.0,0);
		\coordinate (T3) at (7.2,0);
		\coordinate (T4) at (10.4,0);
		\coordinate (T5) at (13.6,0);
		\coordinate (TFused) at (13.0,0);

		\foreach \leftdot/\rightdot in {T1/T2,T2/T3,T3/T4,T4/TFused}{%
			\draw[step boundary] ($(\leftdot)!0.5!(\rightdot)+(0,0.18)$) -- ++(0,4.72);
		}
		
		\TimelineRowFour{4.5}{Sagittal}{0.8}{4.0}{7.2}{10.4}{sa}{sagittalcolor}
		\TimelineRowFour{3.5}{Coronal}{0.8}{4.0}{7.2}{10.4}{co}{coronalcolor}
		\draw[draw=sagittalcolor, line width=0.85pt] (10.4,4.5) .. controls (11.6,4.3) and (12.1,4.) .. (13.0,2.25);
		\draw[draw=coronalcolor, line width=0.85pt] (10.4,3.5) .. controls (11.7,3.4) and (12.0,3.3) .. (13.0,2.2);
		
		\BrainImageOne{0.}{2.00}{0.86cm}
		\BrainImageTwo{0.}{2.40}{0.86cm}
		\BrainImageThree{0.}{2.80}{0.86cm}
		\node[step label] at (0.,1.6) {B volumes};
		\draw[pipe arrow] (0.6,2.32) -- (1.35,2.32);
		\TokenList{1.55}{2.85}{0.74}
		\node[font=\scriptsize, inner sep=0pt, text=gray!45!black, rotate=-55] at (1.8,2.18) {$\cdots$};
		\TokenListAlt{2.1}{2.30}{0.74}
		
		\TokenList{3.1}{2.30}{0.76}
		\TokenListAlt{4.0}{2.78}{0.76}
		\TokenList{4.9}{1.77}{0.76}
		\draw[mutual score arrow] (3.3,2.28) -- (3.8,2.67);
		\draw[mutual score arrow] (3.3,1.80) -- (4.7,1.46);
		\draw[mutual score arrow] (4.2,2.60) -- (4.7,1.66);
		
		\DrawCube{6.0}{1.96}{0.62}{0.64}{0.30}{0.6}{blue!35}
		\DrawFineSlab{6.00}{2.00}{0.62}{0.10}{0.30}{0.6}
		\DrawWireCube{7.12}{1.58}{1.20}{1.22}{0.48}{0.90}
		\DrawCube{7.88}{2.20}{0.30}{0.32}{0.14}{0.80}{blue!40}
		\DrawCube{7.96}{1.58}{0.30}{0.32}{0.14}{0.80}{blue!40}
		\draw[score arrow] (6.59,2.12) -- (7.82,2.35);
		\draw[score arrow] (6.59,2.08) -- (7.90,1.75);
		
		\node[inner sep=0pt] at (10.48,2.26)
		{\includegraphics[width=1.85cm]{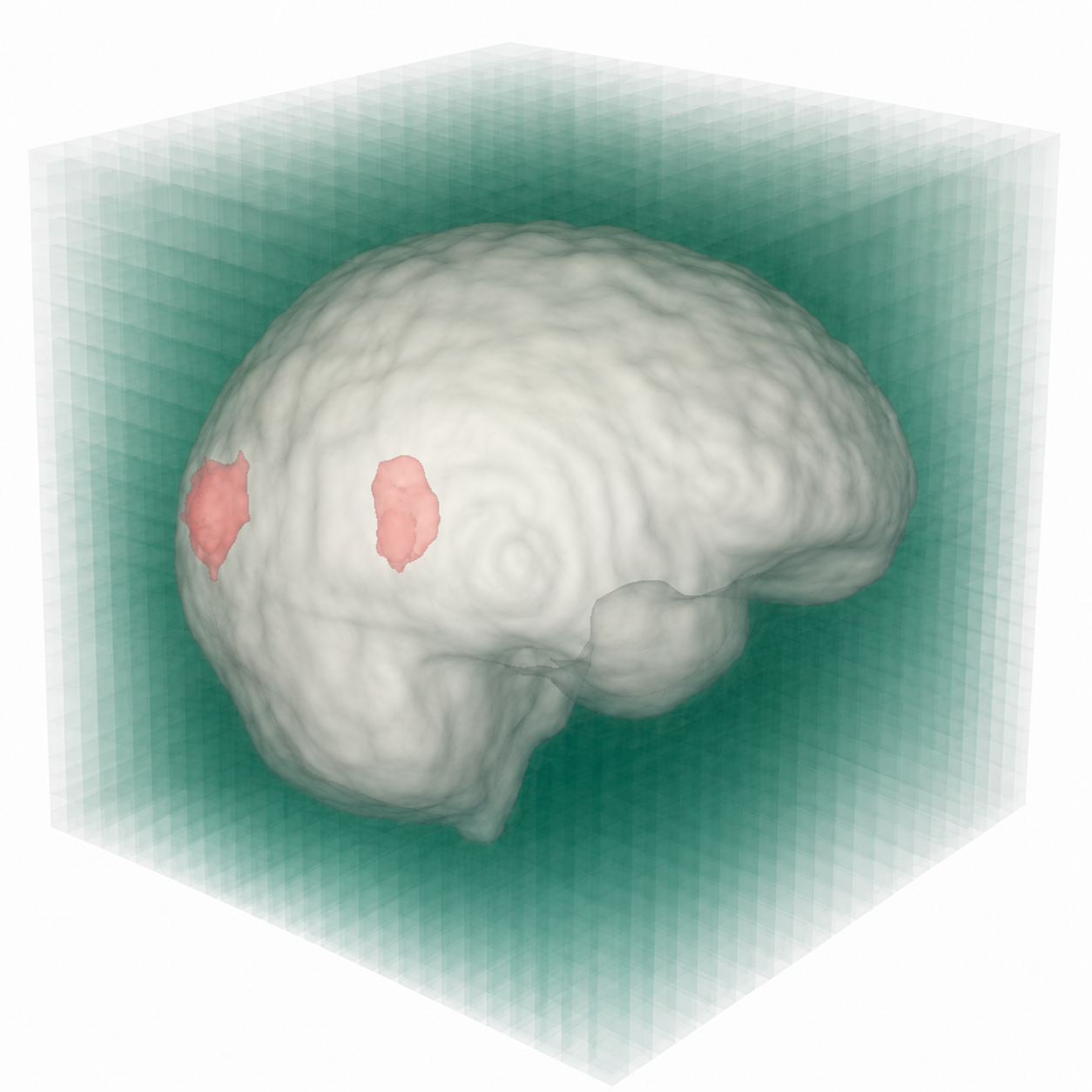}};
		
		\node[inner sep=0pt] at (14.52,2.58)
		{\includegraphics[width=2.25cm]{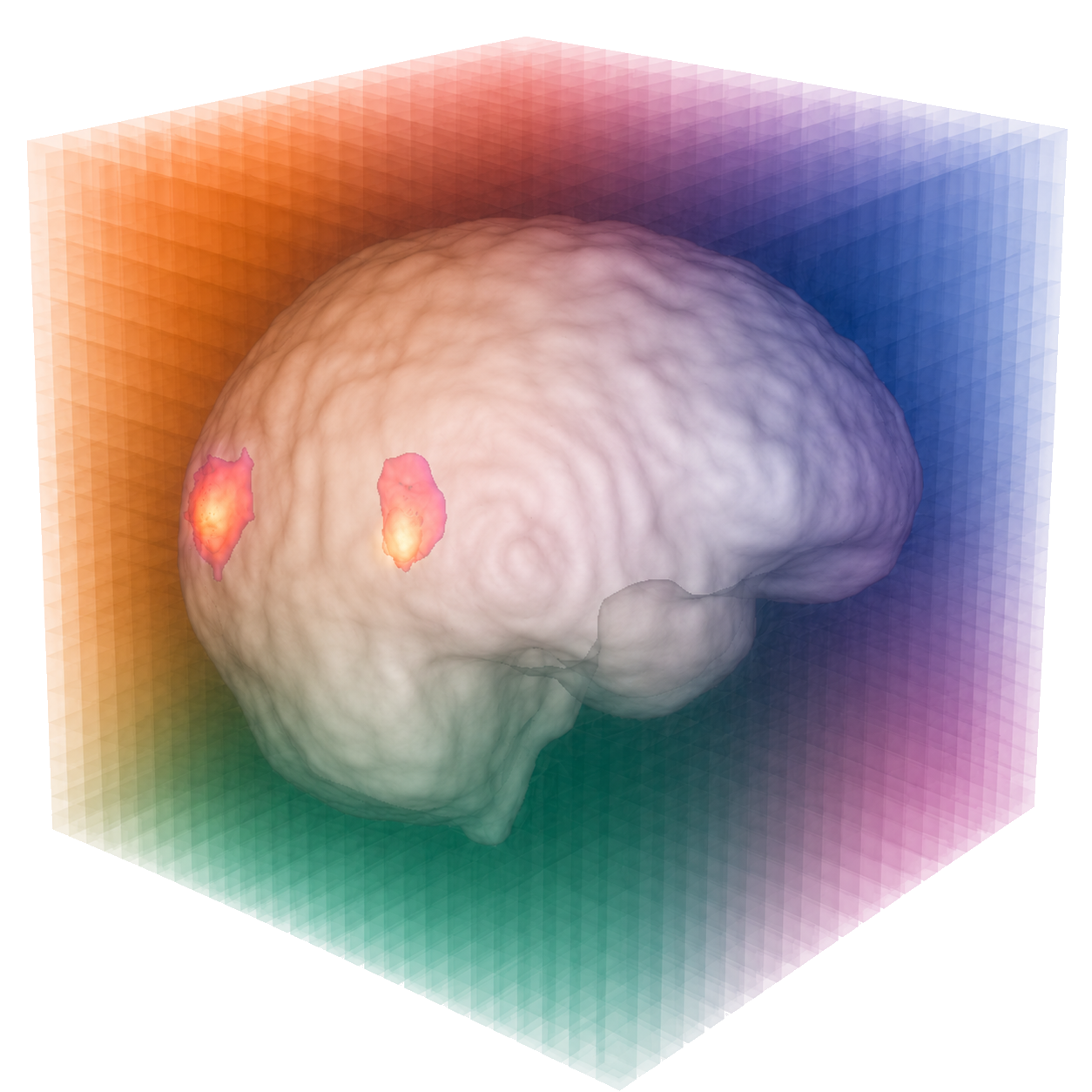}};
		
		\TimelineRowFour{0.85}{Axial}{0.8}{4.0}{7.2}{10.4}{ax}{axialcolor}
		\coordinate (fuseddot) at (13.0,2.2);
		\draw[draw=axialcolor, line width=0.85pt] (ax4) .. controls (11.5,1.0) and (12.1,1.2) .. (fuseddot);
		
			\node[step label] at (0.8,0.46) {Axis-wise \\Tokenization\\(Sec.~\ref{sec:tokenization})};
		\node[step label] at (4.0,0.46) {Cross-subject\\mutual similarity scoring\\ (Sec.~\ref{sec:scoring})};
			\node[step label] at (7.2,0.46) {Coarse-to-Fine \\routing \\(Sec.~\ref{sec:c2f})};
		\node[step label] at (10.4,0.46) {Axis-specific \\ anomaly maps};
		\node[step label] at (13.6,1.50) {Fused 3D \\ anomaly map};
		\node[step dot] at (fuseddot) {};
		
	\end{tikzpicture}
	\caption{Overview of CS3F. Each volume is processed independently along the sagittal, coronal, and axial axes. Slices are encoded by a frozen 2D foundation model and aggregated into volumetric tokens via depth pooling and random projection. The tokens are scored using cross-subject mutual similarity scoring. Axis-specific anomaly maps are fused to obtain the final voxel-level anomaly map.}
	\label{fig:cs3f_overall_pipeline}
\end{figure*}

%% file: tokenization.tex
\begin{figure}[!t]
	\centering
	\begin{tikzpicture}[
		scale=0.55,
		x={(1cm,0cm)},
		y={(0cm,1cm)},
		z={(0.55cm,0.33cm)},
		line join=round,
		line cap=round,
		>=Latex,
		volume edge/.style={black, line width=0.75pt},
		slice plane/.style={
			fill=blue!10,
			draw=blue!80!black!70,
			line width=1.1pt,
			fill opacity=0.45
		},
		token edge/.style={blue!80!black!70, line width=1.15pt},
		token plane/.style={
			fill=blue!35,
			draw=blue!80!black!70,
			line width=1.05pt,
			fill opacity=0.85,
			draw opacity=1
		},
		token plane dashed/.style={
			fill=blue!30,
			draw=blue!80!black!70,
			dashed,
			line width=1.05pt,
			fill opacity=0.85,
			draw opacity=1
		},
		token plane strong/.style={
			fill=blue!10!black,
			draw=blue!80!black!70,
			dashed,
			line width=1.25pt,
			fill opacity=0.45,
			draw opacity=1
		},
		note/.style={
			align=center,
			font=\small
		},
		axis arrow/.style={
			-{Latex[length=2.2mm]},
			black,
			line width=0.9pt
		},
		axis label/.style={
			font=\scriptsize,
			align=center,
			inner sep=1pt
		},
		step dot/.style={
			circle,
			fill=black,
			draw=black,
			inner sep=0pt,
			minimum size=4.4pt
		},
		arrow dot/.style={
			single arrow,
			single arrow head extend=3.pt,
			fill=black,
			draw=black,
			inner sep=0pt,
			minimum size=4.4pt
		},
		step label/.style={
			align=center,
			font=\small,
			text=gray!35!black
		},
		embedding bar/.style={
			rounded corners=3pt,
			draw=gray!55!black,
			line width=1.1pt,
			top color=blue!55,
			bottom color=blue!25,
			middle color=cyan!25,
			shading angle=90
		},
		short embedding bar/.style={
			rounded corners=3pt,
			draw=gray!55!black,
			line width=1.1pt,
			top color=blue!85!pink!30,
			bottom color=blue!55!pink!30,
			middle color=cyan!55!pink!30,
			shading angle=90
		},
		]
		
		\def\W{4.2}   
		\def\H{5.0}   
		\def\D{4.2}   
		
		\coordinate (V000) at (0,0,0);
		\coordinate (V100) at (\W,0,0);
		\coordinate (V110) at (\W,\H,0);
		\coordinate (V010) at (0,\H,0);
		
		\coordinate (V001) at (0,0,\D);
		\coordinate (V101) at (\W,0,\D);
		\coordinate (V111) at (\W,\H,\D);
		\coordinate (V011) at (0,\H,\D);
		
		\draw[volume edge] (V001) -- (V101) -- (V111) -- (V011) -- cycle;
		\draw[volume edge] (V000) -- (V001);
		\draw[volume edge] (V100) -- (V101);
		\draw[volume edge] (V110) -- (V111);
		\draw[volume edge] (V010) -- (V011);

		\begin{scope}
			\node[
				inner sep=0pt,
				opacity=0.8
			]
			at ($(V000)!0.5!(V111)$)
			{\includegraphics[width=2.25cm]{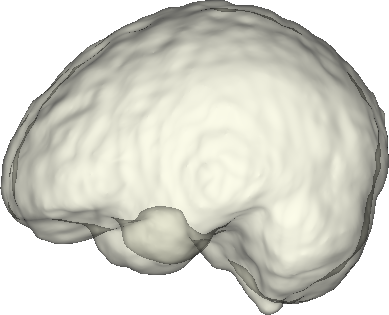}};
		\end{scope}
		
		\draw[volume edge] (V000) -- (V100) -- (V110) -- (V010) -- cycle;
		
		\def\Ys{0.9}
		
		\coordinate (S0) at (0,\Ys,0);
		\coordinate (S1) at (\W,\Ys,0);
		\coordinate (S2) at (\W,\Ys,\D);
		\coordinate (S3) at (0,\Ys,\D);
		
		\filldraw[slice plane] (S0) -- (S1) -- (S2) -- (S3) -- cycle;
		
		\node[note, blue!60!black, anchor=west, font=\scriptsize]
		at (0.,3.45,0.1)
		{Slice-wise ViT\\feature extraction};
		
		\draw[->, gray!35!black,, line width=0.8pt]
		(2.05,2.8,0.1)
		-- (2.,\Ys,1.35);
		
		\def\xo{3.1}
		\def\yo{0.}
		\def\zo{0.}
		\def\sx{1.1}
		\def\sy{1.3}
		\def\sz{1.1}
		
		\coordinate (P000) at (\xo,\yo,\zo);
		\coordinate (P100) at (\xo+\sx,\yo,\zo);
		\coordinate (P110) at (\xo+\sx,\yo+\sy,\zo);
		\coordinate (P010) at (\xo,\yo+\sy,\zo);
		
		\coordinate (P001) at (\xo,\yo,\zo+\sz);
		\coordinate (P101) at (\xo+\sx,\yo,\zo+\sz);
		\coordinate (P111) at (\xo+\sx,\yo+\sy,\zo+\sz);
		\coordinate (P011) at (\xo,\yo+\sy,\zo+\sz);
		
		\def\Ya{0.30}
		\def\Yb{0.60}
		\def\Yc{0.90}
		
		\filldraw[token plane dashed]
		(\xo,\Ya,\zo) --
		(\xo+\sx,\Ya,\zo) --
		(\xo+\sx,\Ya,\zo+\sz) --
		(\xo,\Ya,\zo+\sz) -- cycle;
		
		\filldraw[token plane dashed]
		(\xo,\Yb,\zo) --
		(\xo+\sx,\Yb,\zo) --
		(\xo+\sx,\Yb,\zo+\sz) --
		(\xo,\Yb,\zo+\sz) -- cycle;
		
		\draw[token plane strong]
		(\xo,\Yc,\zo) --
		(\xo+\sx,\Yc,\zo) --
		(\xo+\sx,\Yc,\zo+\sz) --
		(\xo,\Yc,\zo+\sz) -- cycle;
		
		\draw[token edge] (P001) -- (P101) -- (P111) -- (P011) -- cycle;
		\draw[token edge] (P000) -- (P100) -- (P110) -- (P010) -- cycle;
		\draw[token edge] (P000) -- (P001);
		\draw[token edge] (P100) -- (P101);
		\draw[token edge] (P110) -- (P111);
		\draw[token edge] (P010) -- (P011);
		
		\coordinate (TokenOut) at (\xo+\sx+2.,\yo+0.62,\zo+\sz+0.05);
		\coordinate (LongEmbedding) at (7.15,-0.55);
		\coordinate (ShortEmbedding) at (9.55,-0.18);
		
		\node[single arrow, draw=black, fill=white, 
		minimum width = 5pt, single arrow head extend=2pt,
		minimum height=10mm] at (TokenOut) {};
		
		\node[note, gray!35!black, anchor=south west, font=\scriptsize] at (\xo+\sx,\yo-2,\zo+\sz)
		{Average pooling\\ $\oplus$ \\ $\ell_2$ normalization};
		
		\path[embedding bar] (8.8,-1.) rectangle (9.8,4.);
		
		\node[
		rotate=-90,
		font=\scriptsize,
		text=gray!35!black
		] at (9.3,1.5) {Embedding};
		
		\coordinate (TokenOut) at (\xo+\sx+6.4,\yo+0.62,\zo+\sz+0.05);
		\coordinate (LongEmbedding) at (7.15,-0.55);
		\coordinate (ShortEmbedding) at (9.55,-0.18);
		
		\node[single arrow, draw=black, fill=white, 
		minimum width = 5pt, single arrow head extend=2pt,
		minimum height=10mm] at (TokenOut) {};
		
		\node[note, gray!35!black, anchor=south west, font=\scriptsize] at (\xo+\sx+5.3,\yo-2.,\zo+\sz)
		{Random \\ Projection\\(d=128)};
		
		\path[short embedding bar] (13.,-0.5) rectangle (14.,2.5);
		
		\node[
		rotate=-90,
		font=\scriptsize,
		text=gray!35!black
		] at (13.5,1.) {Embedding};
		
		\node[
		align=center,
		font=\small,
		text=gray!45!black
		] at (0.21,-0.75) {};

		\coordinate (StepOne) at (3.1,-1.55);
		\coordinate (StepTwo) at (9.3,-1.55);
		\coordinate (StepThree) at (13.0,-1.55);
		\coordinate (StepFour) at (13.5,-1.55);
		
		\draw[black, line width=0.85pt] (StepOne) -- (StepTwo) -- (StepThree);
		\node[step dot] at (StepOne) {};
		\node[step dot] at (StepTwo) {};
		\node[arrow dot] at (StepThree) {};
		\node[step dot] at (StepFour) {};
		
		\node[step label, anchor=north] at ($(StepOne)+(0,-0.25)$)
		{Slice-wise\\patch features};
		\node[step label, anchor=north] at ($(StepTwo)+(0,-0.25)$)
		{Volumetric\\tokens};
		\node[step label, anchor=north] at ($(StepFour)+(0,-0.25)$)
		{Projected\\tokens};
		
	\end{tikzpicture}
	\caption{
		Illustration of the axial volumetric tokenization pipeline. A 3D volume is decomposed into axial slices and encoded by a frozen vision transformer (DINOv2) to extract patch-level features. Patch features corresponding to the same spatial location across neighboring slices are aggregated via average pooling and \(\ell_2\)-normalized to form volumetric tokens. These tokens are subsequently projected into a compact feature space using a random projection matrix.
	}
	\label{fig:tokenization_tikz}
\end{figure}

%% file: c2f.tex
\begin{figure}[!t]
		\centering
		\begin{tikzpicture}[
			scale=0.55,
			x={(1cm,0cm)},
			y={(0cm,1cm)},
			z={(0.55cm,0.33cm)},
			line join=round,
			line cap=round,
			>=Latex,
			volume edge/.style={black, line width=0.75pt},
			coarse fill/.style={
				fill=blue!35,
				draw=blue!80!black!70,
				line width=1.1pt,
				fill opacity=0.35
			},
			coarse edge/.style={
				draw=blue!80!black!70,
				line width=1.1pt
			},
			fine slab/.style={
				fill=orange!30,
				draw=red!70!black!60,
				line width=0.9pt,
				fill opacity=0.75
			},
			route arrow/.style={
				-{Latex[length=2.5mm]},
				draw=red!70!black!60,
				line width=1.pt
			},
			note/.style={align=center},
			smallnote/.style={align=center, font=\scriptsize}
			]
			
			\newcommand{\DrawCuboid}[7]{%
				\coordinate (A) at (#1,#2,#3);
				\coordinate (B) at (#1+#4,#2,#3);
				\coordinate (C) at (#1+#4,#2+#5,#3);
				\coordinate (D) at (#1,#2+#5,#3);
				\coordinate (E) at (#1,#2,#3+#6);
				\coordinate (F) at (#1+#4,#2,#3+#6);
				\coordinate (G) at (#1+#4,#2+#5,#3+#6);
				\coordinate (H) at (#1,#2+#5,#3+#6);
				
				\draw[#7] (E)--(F)--(G)--(H)--cycle;
				\draw[#7] (D)--(H);
				\draw[#7] (C)--(G);
				\draw[#7] (A)--(E);
				
				\filldraw[#7] (A)--(B)--(C)--(D)--cycle; 
				\filldraw[#7] (A)--(B)--(F)--(E)--cycle; 
				\filldraw[#7] (B)--(C)--(G)--(F)--cycle; 
				\filldraw[#7] (C)--(D)--(H)--(G)--cycle; 
			}
			
			\def\W{4.2}
			\def\H{5.}
			\def\D{4.}
			
			\def\Qx{0}
			\def\Qy{0}
			\def\Qz{0}
			
			\coordinate (Q000) at (\Qx,\Qy,\Qz);
			\coordinate (Q100) at (\Qx+\W,\Qy,\Qz);
			\coordinate (Q110) at (\Qx+\W,\Qy+\H,\Qz);
			\coordinate (Q010) at (\Qx,\Qy+\H,\Qz);
			
			\coordinate (Q001) at (\Qx,\Qy,\Qz+\D);
			\coordinate (Q101) at (\Qx+\W,\Qy,\Qz+\D);
			\coordinate (Q111) at (\Qx+\W,\Qy+\H,\Qz+\D);
			\coordinate (Q011) at (\Qx,\Qy+\H,\Qz+\D);
			
			\draw[volume edge] (Q001)--(Q101)--(Q111)--(Q011)--cycle;
			\draw[volume edge] (Q000)--(Q001);
			\draw[volume edge] (Q100)--(Q101);
			\draw[volume edge] (Q110)--(Q111);
			\draw[volume edge] (Q010)--(Q011);
			\draw[volume edge] (Q000)--(Q100)--(Q110)--(Q010)--cycle;
			
				\begin{scope}
				\node[
				inner sep=0pt,
				opacity=0.8
				]
				at ($(Q000)!0.5!(Q111)$)
				{\includegraphics[width=2.25cm]{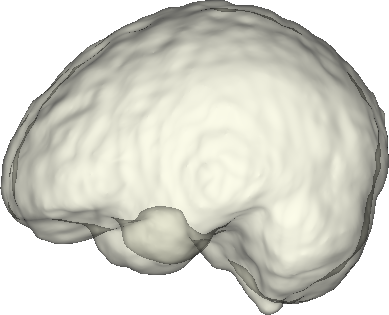}};
			\end{scope}
			
			\node[note, anchor=south] at ($(Q010)!0.5!(Q110)+(1,-7,0)$) {\textbf{Query volume}};
			
			\def\qcSx{1.}
			\def\qcSy{1.2}
			\def\qcSz{1.}
			
			\def\qcX{\Qx+\W-\qcSx}
			\def\qcY{\Qy}
			\def\qcZ{\Qz}
			
			\DrawCuboid{\qcX}{\qcY}{\qcZ}{\qcSx}{\qcSy}{\qcSz}{coarse fill}
			
			\node[smallnote, text=blue!60!black!80, anchor=north east]
			at (\qcX+\qcSx+1,\qcY-0.10,\qcZ+0.12) {coarse token};
			
			\pgfmathsetmacro{\slabH}{\qcSy/7}
			
			\DrawCuboid{\qcX}{\qcY+1*\slabH}{\qcZ}{\qcSx}{\slabH}{\qcSz}{fine slab};
			\draw[draw=blue!80!black!70,
			line width=1.1pt, fill opacity=0.35] (\qcX+\qcSx,\qcY+\qcSy,\qcZ)--(\qcX + \qcSx,\qcY,\qcZ);
			
			\draw[draw=red!70!black!60,
			line width=1.1pt, fill opacity=0.35] (\qcX+\qcSx,\qcY+2*\slabH,\qcZ)--(\qcX + \qcSx,\qcY+\slabH,\qcZ);
			
			\node[smallnote, text=red!70!black, anchor=west]
			at (\qcX-2.4,\qcY+\qcSy-0.8,\qcZ+0.2) {fine token};
			
			\coordinate (QCoarseCenter) at (\qcX+0.8*\qcSx,\qcY+0.3*\qcSy,\qcZ+0.5*\qcSz);
			
			\def\Rx{8.2}
			\def\Ry{0}
			\def\Rz{0}
			
			\coordinate (R000) at (\Rx,\Ry,\Rz);
			\coordinate (R100) at (\Rx+\W,\Ry,\Rz);
			\coordinate (R110) at (\Rx+\W,\Ry+\H,\Rz);
			\coordinate (R010) at (\Rx,\Ry+\H,\Rz);
			
			\coordinate (R001) at (\Rx,\Ry,\Rz+\D);
			\coordinate (R101) at (\Rx+\W,\Ry,\Rz+\D);
			\coordinate (R111) at (\Rx+\W,\Ry+\H,\Rz+\D);
			\coordinate (R011) at (\Rx,\Ry+\H,\Rz+\D);
			
			\draw[volume edge] (R001)--(R101)--(R111)--(R011)--cycle;
			\draw[volume edge] (R000)--(R001);
			\draw[volume edge] (R100)--(R101);
			\draw[volume edge] (R110)--(R111);
			\draw[volume edge] (R010)--(R011);
			\draw[volume edge] (R000)--(R100)--(R110)--(R010)--cycle;
			
			\node[note, anchor=south] at ($(R010)!0.5!(R110)+(1,-7,0)$) {\textbf{Reference volume}};
			
			\def\rOneX{8.0}
			\def\rOneY{0.70}
			\def\rOneZ{0.}
			\DrawCuboid{\rOneX+4.55}{\rOneY}{\rOneZ-2.1}{1.}{1.2}{1.}{coarse fill}
			\coordinate (RC1) at (\rOneX + 3.05,\rOneY,\rOneZ+0.575);
			
			\def\rTwoX{9.9}
			\def\rTwoY{3.30}
			\def\rTwoZ{1.35}
			\DrawCuboid{\rTwoX}{\rTwoY}{\rTwoZ}{1.}{1.2}{1.}{coarse fill}
			\coordinate (RC2) at (\rTwoX-0.5,\rTwoY,\rTwoZ+0.55);
			
			\def\rThrX{10.}
			\def\rThrY{1.00}
			\def\rThrZ{2.30}
			\DrawCuboid{\rThrX}{\rThrY}{\rThrZ}{1.00}{1.20}{1.00}{coarse fill}
			\coordinate (RC3) at (\rThrX-0.5,\rThrY,\rThrZ+0.50);
			
			\node[smallnote, text=blue!60!black!80, anchor=south west]
			at (\Rx+2.,2.5,\Rz+\D+0.15) {top-$L$ \\coarse tokens};
			
			\draw[route arrow] ($(QCoarseCenter) + (0.6, -0.1, 0)$) -- (RC1);
			\draw[route arrow] ($(QCoarseCenter) + (0.6, 0.1, 0)$) -- (RC2);
			\draw[route arrow] ($(QCoarseCenter) + (0.6, 0., 0)$) -- (RC3);
			
		\end{tikzpicture}
		\caption{
			Illustration of the coarse-to-fine routing mechanism. A coarse token from the query volume first identifies its top-$L$ nearest coarse tokens in each reference volume. Fine-scale matching is then performed only within the fine tokens belonging to these selected coarse regions, avoiding exhaustive search over the entire reference volume.
		}
		\label{fig:c2f_routing}
	\end{figure}
	